\documentclass[10pt,journal,compsoc]{IEEEtran}
\ifCLASSOPTIONcompsoc
\usepackage[nocompress]{cite}
\else
\usepackage{cite}
\fi
\ifCLASSINFOpdf
\else
\fi

\usepackage{algorithm}
\usepackage{algorithmic}

\usepackage{epsfig}
\usepackage{graphicx}
\usepackage{amsmath}
\usepackage{amssymb}
\usepackage{tabularx}
\usepackage{booktabs}
\usepackage{amsthm}
\usepackage{color}
\usepackage{amsfonts}

\usepackage{enumitem}
\usepackage{lipsum}
\usepackage{gensymb}

\newcommand{\renata}[1]{\textcolor{black}{#1}}
\newcommand{\comment}[1]{}

\newcommand{\df}{\mathop{=}\limits_{df}}
\newcommand{\mF}{\mathcal{F}}
\newcommand{\hmF}{\hat{\mathcal{F}}}
\newcommand{\norm}[1]{\left|#1\right|}

\begin{document}
	%
	\title{\renata{Isometric Transformation Invariant Graph-based Deep Neural Network}}	
	%
	%
	%
	%
	
	\author{Renata~Khasanova 
		and~Pascal~Frossard
		\IEEEcompsocitemizethanks{\IEEEcompsocthanksitem R. Khasanova and P. Frossard are with the Department
			of Electrical and Computer Engineering, EPFL, Lausanne,
			Switzerland, 1015.\protect\\
			E-mail: renata.khasanova@epfl.ch
		}
	}

	\IEEEtitleabstractindextext{%
		\begin{abstract}
			Learning transformation invariant representations of visual data is an important problem in computer vision. Deep convolutional networks have demonstrated remarkable results for image and video classification tasks. However,  they have achieved only limited success in the classification of images that undergo geometric transformations. In this work we present a novel Transformation Invariant Graph-based Network (TIGraNet), which learns graph-based features that are inherently invariant to isometric transformations such as rotation and translation of input images. In particular, images are represented as signals on graphs, which permits to replace classical convolution and pooling layers in deep networks with graph spectral convolution and dynamic graph pooling layers that together contribute to invariance to isometric transformation. Our experiments show high performance on rotated and translated images from the test set compared to classical architectures that are very sensitive to transformations in the data. The inherent invariance properties of our framework provide key advantages, such as increased resiliency to data variability and sustained performance with limited training sets. Our code is available online.
		\end{abstract}
		
		}

	\maketitle

	\IEEEdisplaynontitleabstractindextext

	%
	\IEEEpeerreviewmaketitle

	\IEEEraisesectionheading{\section{Introduction}\label{sec:introduction}}
	
	\IEEEPARstart{D}{eep} convolutional networks (ConvNets) have achieved impressive results for various computer vision tasks, such as image classification~\cite{bb:krizhevsky2012imagenetNIPS2012} and segmentation~\cite{bb:seg}. However, they still suffer from the potentially high variability of data in high-dimensional image spaces. In particular, ConvNets that are trained to recognize an object from a given perspective or camera viewpoint, will likely fail when the viewpoint is changed or the image of the object is simply rotated. In order to overcome this issue the most natural step is to extend the training dataset with images of the same objects but seen from different perspectives. This however increases the complexity of data collection and more importantly leads to the growth of the training dataset when the variability of the data is high.
	
	Instead of simply augmenting the training set, which may not always be feasible, one can try to solve the aforementioned problem by making the classification architecture invariant to transformations of the input signal as illustrated in Fig.~\ref{fig:adv}. In that perspective, we propose to represent input images as signals on the grid graph instead of simple matrices of pixel intensities. The benefits of this representation is that graph signals do not carry a strict notion of orientation, while at the same time, signals on a grid graph stay invariant to translation. We exploit these properties to create features that are invariant to isometric transformations and we design new graph-based convolutional and pooling layers, which replace their counterparts used in the classical deep learning settings. This permits preserving the transformation equivariance of each intermediate feature representation under both translation and rotation of the input signals. Specifically, our convolutional layer relies on filters that are polynomials of the graph Laplacian for effective signal representation without computing eigendecompositions of the graph signals. We further introduce a new statistical layer that is placed right before the first fully-connected layer of the network prior to the classification. This layer is specific to our graph signal representation, and in turn permits combining the rotation and translation invariance features along with the power of fully-connected layers that are essential for solving the classification task. 
	\begin{figure}
		\centering
		\begin{tabular}{cc}
			\raisebox{0.2cm}{\rotatebox{90}{Conv}} &
			\includegraphics[width=0.8\linewidth]{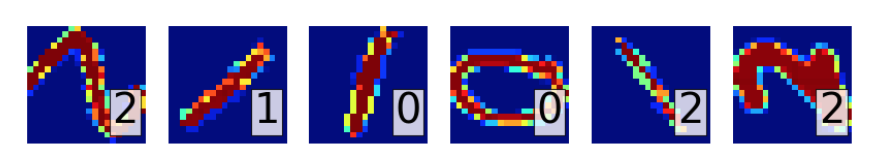} \\
			\raisebox{0.2cm}{\rotatebox{90}{STN}} &
			\includegraphics[width=0.8\linewidth]{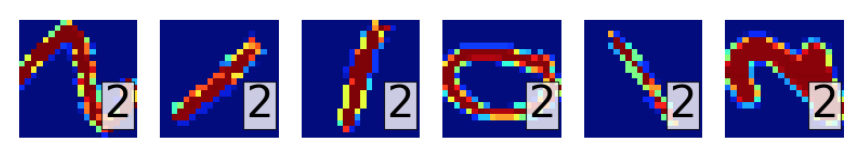} \\
			\raisebox{0.cm}{\rotatebox{90}{TIGraNet}} &
			\includegraphics[width=0.8\linewidth]{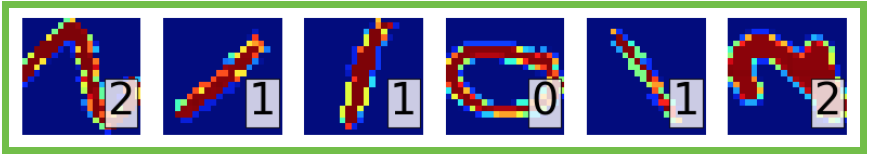} \\
		\end{tabular}
		\caption{Illustrative transformation-invariant handwritten digit classification task. Rotated test images, along with their classification label obtained from ConvNets (Conv)~\cite{bb:lecun}, Spatial-Transformer Network (STN)~\cite{bb:STN}, and our method. (best seen in color)}
		\label{fig:adv}
	\end{figure}
	We finally design a complete architecture for a deep neural network called TIGraNet\footnote{The code is available online:~\text{https://github.com/LTS4/TIGraNet}}, which efficiently combines spectral convolutional, dynamic pooling, statistical and fully-connected layers to process images represented on grid graphs. We train our network in order to learn isometric transformation invariant features. These features are used in sample transformation-invariant image classification tasks, where our solution outperforms the state-of-the-art algorithms for handwritten digit recognition and classification of objects seen from different viewpoints.
	%
	%
	
	\section{Related work}
	\label{s:related_work}
	
	Most of the recent architectures~\cite{bb:lecun98gradient, bb:krizhevsky2012imagenetNIPS2012} have been very successful in processing natural images, but not necessarily in properly handling geometric transformations in the data. We describe below some of the recent attempts that have been proposed to construct transformation-invariant architectures. We further review quickly the recent works that extend deep learning data represented on graphs or networks.

	\subsection{Transformation-invariant deep learning}
	
	One intuitive way to make the classification architectures more robust to isometric transformations is to augment the training set with transformed data~\cite{bb:van2012art}, which however, increases both the training set and training time. Alternatively, there have been works~\cite{bb:fasel2006rotation, bb:Coors2018VISAPP, bb:dima} that incorporate data augmentation inside the deep learning framework and simultaneously process the original image and its transformed versions. While effective, this class of methods does not make the learned features truly transformation invariant, as the network cannot generalize to the image transformations that were not exposed to the network at training time.
	
	A different direction was taken by~\cite{bb:STN, bb:dai2018, bb:marcos2016learning}, who suggest to either learn the image transformation~\cite{bb:STN, bb:dai2018} or use the rotated filter banks~\cite{bb:marcos2016learning} to learn a latent feature representation for an image, independently of the transformation applied to it. The former methods, however, still require a lot of data augmentation to achieve their full potential, and the latter can be strictly invariant only to a limited number of image transformations.
	
	
	Another set of methods~\cite{bb:cohen2016group, bb:dieleman2015rotation, bb:dieleman-cyclic-2016} suggest various deep learning architectures that are either invariant to a small subset of transformations~\cite{bb:cohen2016group, bb:dieleman-cyclic-2016} or designed for a specific problem, such as galaxy morphology prediction~\cite{bb:dieleman2015rotation}. However, similarly to the previous methods, these approaches still need to be trained on a large dataset of randomly rotated images in order to be transformation invariant and achieve effective performance.
	
	
	Contrary to the previous methods, we build on~\cite{bb:KhasanovaF17} and propose to directly learn feature representations that are invariant to isometric data transformations. With such features, our architecture preserves all the advantages of deep networks, but additionally provides invariance to isometric geometric transformations. The methods in~\cite{bb:oyallon2015deep, bb:bruna2013invariant, bb:harm} are the closest in spirit to ours. In order to be invariant to local transformation, the works in~\cite{bb:oyallon2015deep, bb:bruna2013invariant} propose to replace the classical convolutional layers with wavelets, which are stable to some deformations. The latter achieves high performance on texture classification task, however it does not improve the performance of supervised ConvNets on natural images, due to the fact that the final feature representations are too rigid and unable to adapt to a specific task. 
	
	Further,~\cite{bb:rev1, bb:Cohen2018, bb:rev2} propose to use convolutional filters in Fourier domain to reduce complexity. The work~\cite{bb:rev2} introduces spectral pooling to truncate the representation in the frequency domain and~\cite{bb:Cohen2018} introduces a rotation-equivariant spherical cross-correlation layer that makes the network learn rotation-equivariant features. This method, however, is not transformation invariant and is tailored for a specific set of tasks, where signals are defined on a sphere. 
	
	Finally, work~\cite{bb:harm} proposes a so called Harmonic Network, which uses specifically designed complex valued filters to make feature representations equivariant to rotations. The work~\cite{bb:3dSteerableWeiler} then presents a 3D analog of~\cite{bb:harm}  and introduces 3D Steerable CNN that is equivariant to rigid body motions. These method, however, still requires the training dataset to contain examples of rotated images to achieve its full potential. On the other hand, we propose building features that are inherently invariant to isometric transformations, which allows us to train more compact networks and achieve state-of-the-art results.
	
	\renata{We extend our earlier work~\cite{bb:KhasanovaF17}, which introduces a novel graph-based deep learning architecture that learns isometric transformation invariant features. We theoretically show that the spectral convolutional filters defined in~\cite{bb:KhasanovaF17} are equivariant to graph isometric transformations of the input signal. We further introduce the quasi-equivariance property and show that the same filters are quasi-equivariant to a broader class of general isometric transformations. We finally support these theoretical results with additional experiments.}

	\subsection{Deep learning and graph signal processing}
	
	While there has been a lot of research efforts related to the application of deep learning methods to traditional data like 1-D speech signals or 2-D images, it is recently that researchers have started to consider the analysis of network or graph data with such architectures \cite{bb:kipf2016semi, bb:henaff2015deep, bb:nips_fingerprint, bb:Structural-RNN}.
	The work in \cite{bb:bruna-iclr-14} has been among the pioneering efforts in trying to bridge the gap between graph-based learning and deep learning methods. The authors calculate the projection of graph signals onto the space defined by the eigenvectors of the Laplacian matrix of the input graph, which itself describes the geometry of the data. It however requires an expensive calculation of the graph eigendecomposition, which can be a strong limitation for large graphs, as it requires $O(N^3)$ operations with $N$ being the number of nodes in the graph. The authors in \cite{bb:Mikhael} later propose an alternative to analyse network data, which is built on a vertex domain feature representation and on fast spectral convolutional filters. 
	
	Finally, the authors in work \cite{bb:MontiBMRSB17} propose a more general framework for convolutional deep networks operating in the non-Euclidean domains. However, these methods directly integrate the learned graph-based features into a fully-connected layer similarly to classical ConvNets, which makes the resulting feature representation dependent on the transformation of the input signal. Therefore in order to achieve transformation invariance such methods require all signal transformations to be exposed to the network at training time, which makes the training process time-consuming and inefficient.
	
	Some recent works further apply deep networks to particular graph data analysis tasks~\cite{bb:bronstein2017geometric}. For example, the authors in~\cite{bb:bronsteingeodesicconv} generalize the ConvNets paradigm to the extraction of feature descriptors for 3D shapes that are defined on different graphs. The work in \cite{bb:nips_fingerprint} further applies deep architectures to train descriptors of chemical molecules, which can be used to predict properties of novel molecules. In~\cite{bb:deepwalk}, the  authors introduce deep networks to analyze web-scale graphs using random walks, which can be used for social network classification tasks. The above algorithms are however specifically developed for a particular task, therefore their generalization to other problems is often difficult.
	
	To the best of our knowledge, these approaches to deep learning on graphs do not provide true transformation-invariance in image classification. At the same time, the methods that specifically target transformation invariance in image datasets mostly rely on data augmentation, which largely remains an art. We propose to bridge this gap and present a novel method that uses the power of graph signal processing to add translation and rotation invariance to the image feature representation learned by deep networks.
	

	\section{Graph signal processing elements}
	\label{s:gsp}
	
	We now briefly review some elements of graph signal processing that are important in the construction of our novel framework. We represent an input image as a signal $y(v_n)$ on the nodes $\{v_n\}$ of the grid graph $G$. In more details, $G=\{{\mathcal{V},\mathcal{E}}, A\}$ is an undirected, weighted and connected graph, where $\mathcal{V}$ is a set of $N$ vertices (i.e., the image pixels), $\mathcal{E}$ is a set of edges and $A$ is a weighted adjacency matrix. An edge $e(v_i,v_j)$ that connects two nodes $v_i$ and $v_j$ is associated with the weight $w(v_i, v_j)=w(v_j, v_i)$, which is usually chosen to capture the distance between both vertices. The edge weight is set to zero for pairs of nodes that are not connected, and all the edge weights together build the adjacency matrix $A$. Every vertex $v_n$ of $G$ carries the luminance value of the corresponding image pixel. Altogether, the valued vertices define a graph signal $y(v_n): \mathcal{V} \to \mathbb{R}$. 
	
	Similarly to regular 1-D or 2-D signals, the graph signals can be efficiently analysed via harmonic analysis and processed in the spectral domain~\cite{bb:shuman2013emerging}. In that respect, we first consider the normalized graph Laplacian operator of the graph $G$, defined as 
			\begin{equation}
	\mathcal{L} =
	\begin{cases}
	1\;,  & \mbox{if } i = j, \; d(v_i) \neq 0 \;, \\
	\big(d(v_i) d(v_j)\big)^{-1/2}\;, &  \mbox{if } i \neq j, \; v_i \sim v_j \;, \\
	0\;, &  \mbox{otherwise}\;, \\
	\end{cases}
	\label{eq:laplacian}
	\end{equation} 
	\noindent
	where $v_i \sim v_j$ denotes that node $v_i$ is adjacent to the node $v_j$, and $d(v_i)$ is the degree of the vertex $v_i$, computed as 
	\begin{equation}
	d(v_i) = \sum_{j=0, j \neq i}^N w(v_i, v_j)\;.
	\end{equation}
	The Laplacian operator is a real symmetric and positive semidefinite matrix, which has a set of orthonormal eigenvectors and corresponding eigenvalues. Let $\chi=[\chi_0, \chi_1, \dots, \chi_{N-1}]$ denote these eigenvectors and $\{0=\lambda_0 \leq \lambda_1\leq \dots \leq \lambda_{N-1} \}$ denote the corresponding eigenvalues with $\lambda_{N-1} = \lambda_{\mathrm{max}} = 2$ for the normalized Laplacian $\mathcal{L}$. The eigenvectors form a Fourier basis and the eigenvalues carry a notion of frequencies as in the classical Fourier analysis. The Graph Fourier Transform $\hat{y}(\lambda_i)$ at frequency $\lambda_i$ for signal $y$ and respectively the inverse graph Fourier transform for the vertex $v_n \in \mathcal{V}$ are thus defined as:
	\begin{equation}
	\hat{y}(\lambda_i) = \sum_{n=1}^{N} y(v_n) \chi_i^*(v_n),
	\end{equation}
	and 
	\begin{equation}
	y(v_n) = \sum_{i=0}^{N-1} \hat{y}(\lambda_i) \chi_i(v_n).
	\end{equation}
	\noindent

	Equipped with the above notion of Graph Fourier Transform, we can denote the generalized convolution of two graph signals $y_1$ and $y_2$ with help of the graph Laplacian eigenvectors as
	\begin{equation}
	(y_1 * y_2)(v_n) = \sum_{i=0}^{N-1} \hat{y_1}(\lambda_i) \hat{y_2}(\lambda_i) \chi_i (v_n).
	\label{eq:conv}
	\end{equation}
	By comparing the previous relations, we can see that the convolution in the vertex domain is equivalent to the multiplication in the graph spectral domain. Graph spectral filtering can further be defined as
	\begin{equation}
	\hat{y}_f(\lambda_i) = \hat{y}(\lambda_i) \hat{h}(\lambda_i), 
	\end{equation}
	where $\hat{h}(\lambda_i)$ is the spectral representation of the graph filter $h(v_n)$ and $\hat{y}_f(\lambda_i)$ is the Graph Fourier Transform of the filtered signal $y_f$. In a matrix form, the graph filter can be denoted by $H \in \mathbb{R^{N \times N}}: H=\chi \hat{H} \chi^T$, where $\hat{H}$ is a diagonal matrix constructed on the spectral representation of the graph filter:
	\begin{equation}
	\hat{H}=\mathrm{diag} (\hat{h}(\lambda_0), \dots, \hat{h}(\lambda_{N-1})).
	\label{eq:hatH}
	\end{equation}
	The graph filtering process becomes $y_f = H y $, with the vectors $y$ and $y_f$ being the graph signal and its filtered version in the vertex domain. Finally, we can define the generalized translation operator $T_{v_n}$ for a graph signal $y$ as the convolution of $y$ with a delta function $\delta_{v_n}$ centered at vertex $v_n$ \cite{bb:thanou2014learning}:
	\begin{equation}
	\begin{array}{rl}
	T_{v_n} y & = \sqrt{N}(y * \delta_{v_n}) \\
	& =\sqrt{N}\sum_{i=0}^{N-1}\hat{y}(\lambda_i)\chi_i^*(v_n)\chi_i .
	\end{array}
	\label{eq:transl}
	\end{equation}
	More details about the above graph signal processing operators can be found in \cite{bb:shuman2013emerging}.
	
	\section{Graph-based convolutional network}
	\label{s:overview}
	
	\begin{figure*}[tb!]
		\includegraphics[width=1\linewidth]{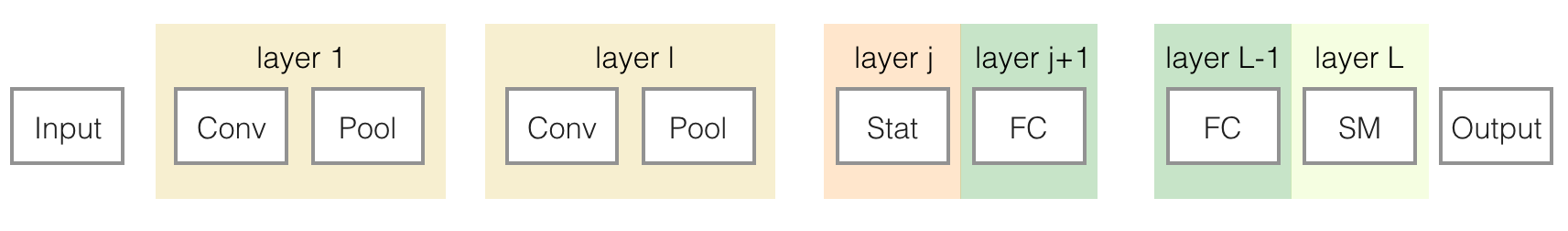}
		\vspace{-1.0cm}
		\caption{{\bf TIGraNet} architecture. The network is composed of an alternation of spectral convolution layers $\mathcal{F}^{l}$  and dynamic pooling layers $\mathcal{P}^{l}$, followed by a statistical layer $\mathcal{H}$, multiple fully-connected layers (FC) and a softmax operator (SM). The input of the network is an image that is represented as a signal $y_0$ on the grid-graph with Laplacian matrix $\mathcal{L}$. The output of the system is a label that corresponds to the most likely class for the input sample.}
		\label{fig:architecture}
	\end{figure*}
	
	We now present the overview of our new architecture, which is illustrated in Fig.~\ref{fig:architecture}. The input to our system can be characterized by a normalized Laplacian matrix $\mathcal{L}$ computed on the grid graph $G$ and the signal $y_0 = (y_0(v_1), \dots, y_0(v_N))$, where $y_0(v_j)$ is the intensity of the pixel $j$ in the input image and $N$ is the number of pixels in the images. Our network eventually returns a class label for each input signal.
	
	In more details, our deep learning architecture consists of an alternation of spectral convolution layers $\mathcal{F}^{l}$ and dynamic pooling layers $\mathcal{P}^{l}$. They are followed by a statistical layer $\mathcal{H}$ and a sequence of fully-connected layers (FC) that precedes a softmax operator (SM) that produces a categorical distribution over labels to classify the input data. Both the spectral convolution and the dynamic pooling layers contain $K_l$ operators denoted by $\mathcal{F}_i^{l}$ and $\mathcal{P}_i^{l}, i=1,\dots,K_l$, respectively.  Each convolutional layer $\mathcal{F}_i^{l}$ is specifically designed to compute transformation-invariant features on grid graphs. The dynamic pooling layer follows the same principles as the classical ConvNet's max-pooling operation but preserves the graph structure in the signal representation. Finally, the statistical layer $\mathcal{H}$ is a new layer designed specifically to achieve invariance to isometric transformations on grid graphs. It does not have any correspondent in the classical ConvNets architectures. We discuss more thoroughly each of these layers in the remainder of this section.
	
	\subsection{Spectral convolutional layer}
	\label{s:conv}
	
	Similarly to the convolutional layers in classical architectures, the spectral convolutional layer $l$ in our network consists of $K_l$ convolutional filters $\mathcal{F}_i^{l}$, as illustrated in Fig.~\ref{fig:conv}. However, each filter $i$ operates in the graph spectral domain. In order to avoid computing the graph eigen-decomposition that is required to perform filtering through Eq. (\ref{eq:conv}), we choose to design our graph filters as smooth polynomial filters of order $M$~\cite{bb:thanou2014learning}, which can be written as
	\begin{equation}
	\hat{h}(\lambda_l) = \sum_{m=0}^M \alpha_m \lambda_l^m .
	\end{equation}
	Following the notation of Eq.~(\ref{eq:hatH}), each filter operator in the spectral convolutional layer $l$ can be written as
	\begin{equation}
	\mathcal{F}_i^{l} = \sum_{m=0}^M \alpha_{i,m}^{l} \mathcal{L}^m ,
	\label{eq:pol_filt_1}
	\end{equation}
	\noindent
	where $\mathcal{L}^m$ denotes the Laplacian matrix of power $m$. The polynomial coefficients $\{\alpha_{i,m}^{l}\}$ have to be learned during the training of the network, for each spectral convolutional layer $l$. Each column of this $N \times N$ operator corresponds to an instance of the graph filter centered at a different vertex of the graph~\cite{bb:thanou2014learning}. The support of each graph filter is directly controlled by the degree $M$ of the polynomial kernel, as the filter takes values only on vertices that are less than M-hop away from the filter center. Larger values of $M$ require more parameters but allow training more complex filters. Therefore, $M$ can be seen as a counterpart of the filter's size parameter in the classical ConvNets.
	
	\begin{figure}[t!]
		\includegraphics[width=1\linewidth]{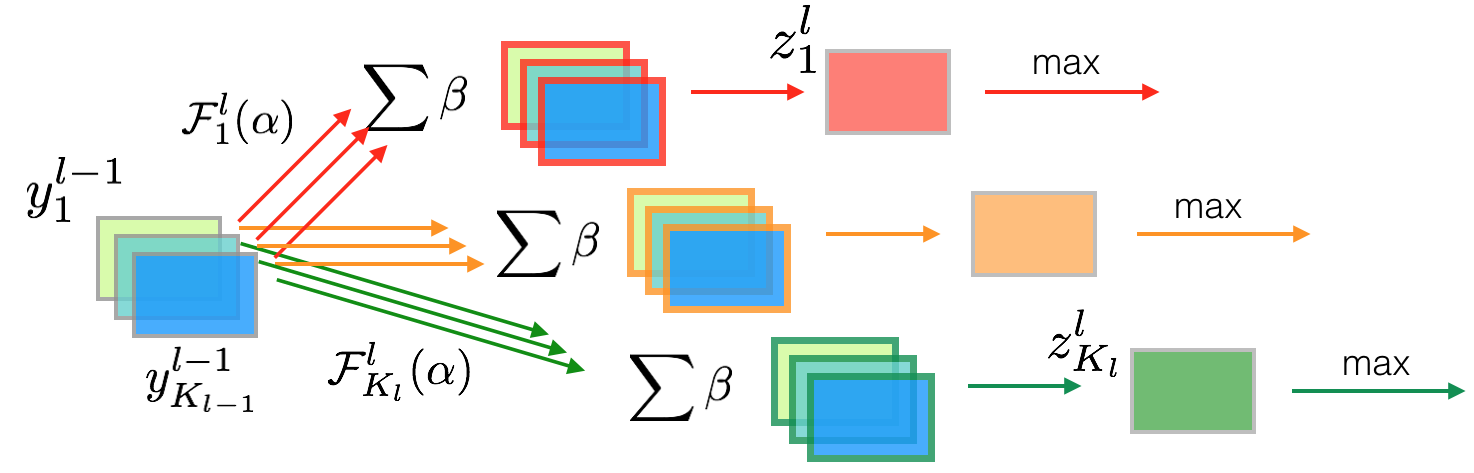}
		\caption{Spectral convolutional layer $\mathcal{F}^{l}$ in {\bf TIGraNet}. The outputs of the previous layer $l-1$ are fed to a set of filter operators $\mathcal{F}_i^{l}$. The outputs of $\mathcal{F}_i^{l}$ are then linearly combined to get the filter maps $z_i^{l}$ that are further passed to the dynamic pooling layer.}
		\label{fig:conv}
	\end{figure}
	
	The filtering operation then simply consists in multiplying the graph signal by the transpose of the operator defined in Eq.~(\ref{eq:pol_filt_1}), namely 
	\begin{equation}
	\tilde{y}_{i,k}^{l} =\left[\mathcal{F}_i^{l}|_{\mathcal{N}_i^{l-1}}\right]^T y_k^{l} , 
	\label{eq:pol_filt_2}
	\end{equation}
	\noindent
	where $y_k^{l}$ and $\tilde{y}_{i,k}^{l}$ are the graph signals at the input and respectively the output of the $l^{th}$ spectral convolutional layer (see Fig.~\ref{fig:conv}). In particular, $y_k^{(1)} = y_0$ is the input image for the first level filter, while at the next levels of the network $y_k^{l}$ is rather one of the feature maps output by the lower layers. We finally use the notation $A |_{\mathcal{N}_i^{l}}$ to represent an operator that preserves the columns of the matrix $A$, which have an index in the set ${\mathcal{N}_i^{l}}$, and set all the other columns to zero. This operator permits computing the filtering operations only on specific vertices of the graphs. It is important to note that the spectral graph convolutional filter permits equivariance to isometric transformations, which is a key property for designing a classifier that is invariant to rotation and translation.

	Finally, the output of the $l^{th}$ spectral convolutional layer is a set of $K_{l}$ feature maps $z_i^{l}$. Each $i^{th}$ feature map is computed as a linear combination of the outputs of the corresponding polynomial filter as follows: 
	\begin{equation}
	z_i^{l} =  \sum_{k=1}^{K_{l-1}} \beta_{k}^{l} \ \tilde{y}_{i,k}^{l} , 
	\label{eq:lincomb}
	\end{equation}
	\noindent
	where the set of signals $\tilde{y}_{i,k}^{l}$ are the outputs of the $i^{th}$ polynomial filter applied on the $K_{l-1}$ input signals of the spectral convolutional layer with Eq. (\ref{eq:pol_filt_2}). The vector of parameters $\{\beta_{k}^{l}\}$, for each spectral convolutional layer $l$ is learned during the training of the network. The operations in the spectral convolutional layer are illustrated in Fig.~\ref{fig:conv}. Lastly, the complexity of spectral filtering can be computed based on the fact that $\mathcal{L}$ and thus the filters are sparse matrices. Then, the complexity is $O(|\mathcal{E}_M | N)$ where $|\mathcal{E}_M|$ is a maximum number of nonzero elements in the columns of $\mathcal{F}_i^{l}$.

	\subsection{Dynamic pooling layer}
	\label{s:pool}
	
	In classical ConvNets the goal of pooling layers is to summarize the outputs of filters for each operator at the previous convolutional layer. Inspired by \cite{bb:dmaxpool} we introduce a novel layer that we refer to as dynamic pooling layer, which basically consists in preserving only the most important features at each level of the network.
	
	In more details, we perform a dynamic pooling operation, which is essentially driven by the set of graph vertices of interest $\Omega^{l}$. This set is initialised to include all the nodes of graph, i.e., $\Omega^{(1)}= \mathcal{V}$. It is then successively refined along the progression through the multiple layers of the network. More particularly, for each dynamic pooling layer $l$, we select the $J_l$ vertices that are part of $\Omega^{l-1}$ and that have the highest values in $z_i^{l}$. The indexes of these largest valued vertices form a set of nodes $\mathcal{N}_i^{l}$. The union of these sets for the different features maps $z_i^{l}$ form the new set $\Omega^{l}$, i.e., 
	\begin{equation}
	\Omega^{l} = \bigcup\limits_{i=1}^{K_{l}}\mathcal{N}_i^{l} .
	\label{eq:omega}
	\end{equation}
	The sets $\Omega^{l}$  drives the pooling operations at the next dynamic pooling layer $\mathcal{P}^{l+1}$. We note that, by construction, the different sets $\Omega^{l}$ are embedded, namely we have $\Omega^{l}  \supseteq \Omega^{l+1}, \ \forall l \in [1..L]$. The Algorithm~\ref{alg:pool} summarize our approach, and Fig.~\ref{fig:pool} illustrates the effect of the pooling process through the different network levels.
	
	\begin{algorithm}[h!]
		\begin{algorithmic}[1]
			\STATE {\bf Input:}\quad \ Feature maps $z_i^{l}, i \in [1,K_l]$
			\STATE \qquad\quad\quad Set of nodes of interest $\Omega^{l-1}$
			\STATE \qquad\quad\quad Number of active nodes, $J_l$
			\vspace{5pt}
			\FOR {$i \in [1,K_l]$} 
			\STATE $\overline{\mathcal{N}_i^{l}} = \mathcal{V}$
			\STATE $\mathcal{N}_i^{l} = \emptyset $
			\FOR {$j \in [1,\max{\left(J_l, |\Omega^{l-1} |\right)}]$} 
			\STATE $\nu =  \mathop{\arg\max} \limits_{v \in \left(\Omega^{l-1} \cap \overline{\mathcal{N}_i^{l}} \right)} z_i^{l}(v)$ 
			\STATE $\overline{\mathcal{N}_i^{l}} = \overline{\mathcal{N}_i^{l}} \setminus \{\nu\}$
			\STATE $\mathcal{N}_i^{l} = \mathcal{N}_i^{l} \cup \{\nu\}$
			\ENDFOR
			\ENDFOR
			\STATE $\Omega^{l}=\bigcup\limits_{i=1}^{K_{l}}\mathcal{N}_i^{l}$
		\end{algorithmic}
		\caption{Dynamic pooling layer at layer $l$.}
		\label{alg:pool}
	\end{algorithm}

	\begin{figure}[!t]
		\centering
		\includegraphics[width=1\linewidth]{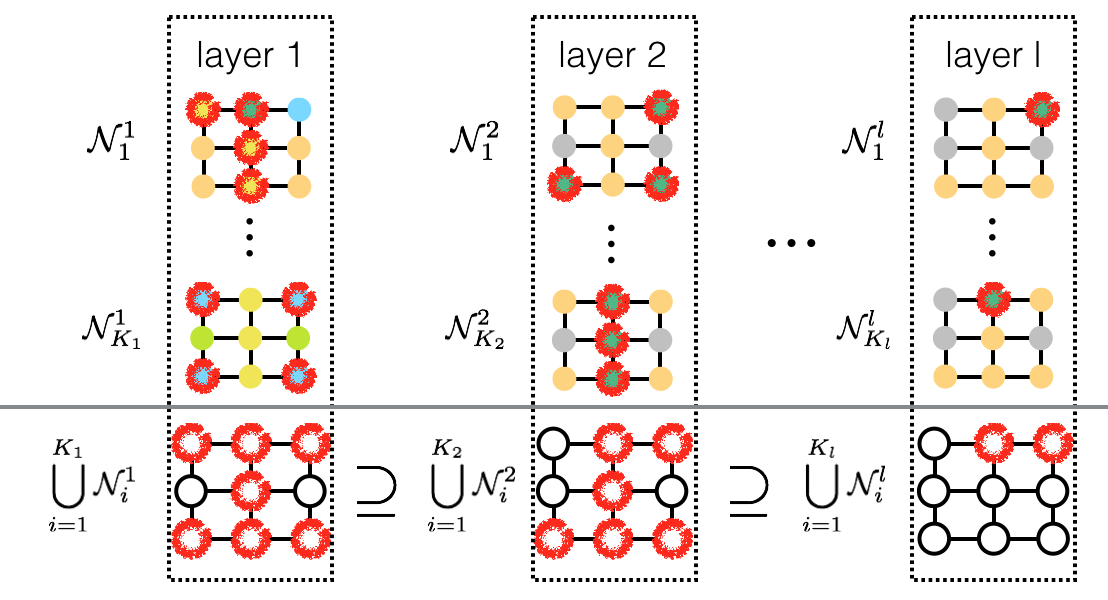} 
		\caption{Pooling process, with succession of dynamic pooling layers with operators $\mathcal{P}_i^{l}$ that each selects the vertices with maximum intensity according to Eq.~(\ref{eq:omega}).
		}
		\label{fig:pool}
	\end{figure}
	
	The sets $\mathcal{N}_i^{l}$ are used to control the filtering process at the next layer. The spectral convolutional filters $\mathcal{F}_i^{l+1}$ compute the output of filters centred on the nodes in $\mathcal{N}_i^{l}$ that are selected by the dynamic pooling layer, and not necessarily for all the nodes in the graph. The filtering operation is given by Eq.~(\ref{eq:pol_filt_2}).
	
	
	Finally, we note that one of the major differences with the classical max-pooling operator is that our dynamic pooling layer is not limited to a small neighbourhood around each node. Instead, it considers the set of nodes of interest $\Omega_l$ which is selected over all graph's nodes. The dynamic pooling operator $\mathcal{P}^{l}$ is thus equivariant to the isometric transformations $R$, similarly to the spectral convolutional layers, which is a key property in building a transformation-invariant classification architecture. The complexity of  $\mathcal{P}^l$ is comparable with the classical pooling operator as the task of $\mathcal{P}^l$ is equivalent to finding $J_l$ highest statistics. Using the selection algorithm~\cite{bb:Knuth98} we can reach the average computational complexity of $O(N)$.
	
	\subsection{Upper layers}
	\label{s:hist}
	
	After the series of alternating spectral convolutional and dynamic pooling layers, we add output layers that compute the label probability distributions for the input images. Instead of connecting directly a fully-connected layer as in classical ConvNet architectures, we first insert a new statistical layer, whose output is then fed into fully-connected layers (see Fig.~\ref{fig:architecture}). 
	
	The main motivation for the statistical layer resides in our objective of designing a transformation-invariant classification architecture. If fully-connected layers are added directly on top of the last dynamic pooling layers, their neurons would have to memorize large amounts of information corresponding to the different positions and rotation of the visual objects. Instead, we propose to insert a new statistical layer, which computes transformation-invariant statistics of the input signal distributions.
	
	In more details, the statistical layer estimates the  distribution of values on the active nodes after the last pooling layer. The inputs of the statistical layer $j$ are denoted as $\tilde{z}_i$, which correspond to the outputs $z_i^{j-1}$ of the last pooling layer $\mathcal{P}^{j-1}$ where the values on non-active nodes (i.e., the nodes in $\overline{\mathcal{N}_i^{l}}$) are set to zero. We then calculate multiscale statistics of these input features maps using Chebyshev polynomials of the graph Laplacian. These polynomials have the advantage of a fast computation due to their iterative construction, and they can be adapted to distributed implementations \cite{bb:shuman2011chebyshev}. In order to construct these polynomials, we first shift the spectrum of the Laplacian $\mathcal{L}$ to the interval $[-1,1]$, which is the original support of Chebyshev polynomials. Equivalently, we set $\tilde{\mathcal{L}} = \mathcal{L} - I $. 
	
	\begin{figure}[!t]
		\centering
		\includegraphics[width=1\linewidth]{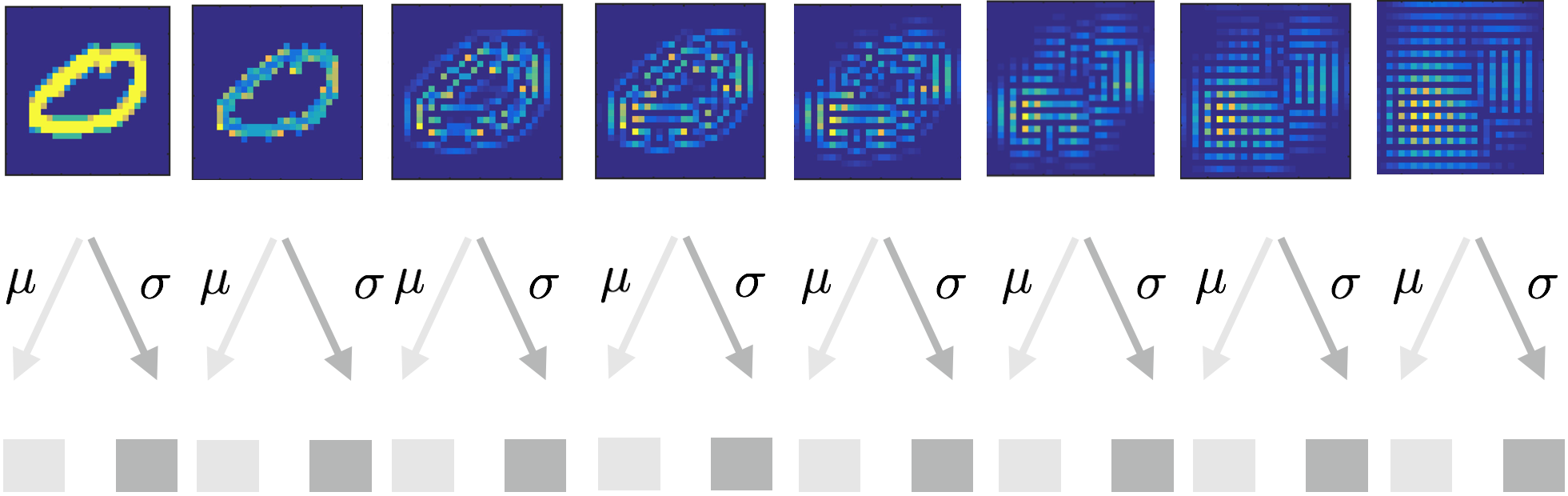} 
		\caption{Illustration of the statistical layer, which calculates multiscale statistics by filtering each input feature map by Chebyshev polynomials filters \cite{bb:shuman2011chebyshev} and taking their second-order statistics. Resulting statistics are vectorized and build an invariant to graph isometric transformation.  
		}
		\label{fig:stat_layer_im}
	\end{figure}
	
	As suggested in \cite{bb:Mikhael}, for each input feature map $\tilde{z}_i$ we iteratively construct a set of signals $t_{i,k}$ using graph Chebyshev polynomials of order $k$, with $k \leq K_{max}$, as
	\begin{equation}
	t_{i,k} = 2 \tilde{\mathcal{L}} t_{i, k-1} - t_{i, k-2},
	\end{equation}
	\noindent
	with $t_{i, 0} = \tilde{z}_i$ and $t_{i, 1}=\tilde{\mathcal{L}} \tilde{z}_i$.
	We finally compute a feature vector that gathers the first order statistics of the magnitude of these signals, namely the mean $\mu_{i,k}$ and variance $\sigma^2_{i,k}$ for each signal $| t_{i,k} |$  (see Fig.~\ref{fig:stat_layer_im}). This forms a feature vector $\phi_{i}$ of $2 K_{max}+2$ elements, i.e., $\phi_{i} = [\mu_{i,0},\sigma^2_{i,0}, \hdots, \mu_{i,K_{max}},\sigma^2_{i,K_{max}} ]$. We choose these particular statistics as they are prone to efficient gradient computation, which is important during back propagation. Furthermore, we note that such feature vectors are inherently invariant to transformation such as translation or rotation.
	
	The feature vectors $\phi_i$'s are eventually sent to a series of fully-connected layers similarly to classical ConvNet architectures. However, since our feature vectors are transformation invariant, the fully-connected layers will also benefit from these properties. This is in opposition to their counterparts in classical ConvNet systems, which need to compute position-dependent parameters. The details about fully-connected layer parameters are given in the Section~\ref{s:exp}. The output of the fully-connected layers is then fed to a softmax layer \cite{bb:bishop:2006softmax}, which finally returns the probability distribution of a given input sample to belong to a given set of classes. 
	
	\subsection{Training}
	We use supervised learning and train our network so that it maximizes the log-probability of estimating the correct class of training samples via logistic regression. Overall, we need to compute the values of the parameters in each convolutional and in fully-connected layers. The other layers do not have any parameter to be estimated.  We train the network using a classical back-propagation algorithm and learn the parameters using ADAM stochastic optimization~\cite{bb:adam}.
	
	We provide more details here about the computation that are specific to our new architecture. We refer the reader to \cite{bb:rumelhart1988learning} for more details about the overall training procedure. The back-propagation in the spectral convolutional layer is performed by evaluating the partial derivatives with respect to the parameters $\alpha: \alpha \in \mathbb{R}^{K_{l-1} \times M}$ of the spectral filters, and to the parameters $\beta: \beta \in \mathbb{R}^{K_{l-1}}$ of the feature map construction. The partial derivatives read
	\begin{equation}
	\frac{\partial E} {\partial \alpha^l_{i,m}} = \sum_{k=0}^{K_{l-1}} \beta_k^l \left[\mathcal{L}^m |_{\mathcal{N}_i^{l-1}}\right] y_{k}^{l-1} \frac{\partial E}{\partial z^{l}_i}, \\
	\label{eq:der1}
	\end{equation}
	\begin{equation}
	\frac{\partial E} {\partial \beta_j^l} =  \sum_{m=0}^M \alpha^l_{i,m} \left[\mathcal{L}^m |_{\mathcal{N}_i^{l-1}}\right] y_{j}^{l-1} \frac{\partial E}{\partial z^{l}_i}, \\
	\label{eq:der2}
	\end{equation}
	\noindent
	where $E$ is the negative log-likelihood cost function, $z^{l}_i=y^{l}_i$ is the output feature map of layer $l$, $K_{l-1}$ denotes the number of feature maps at the previous layer of the network, $M$ is the polynomial degree of the convolutional filter and $\mathcal{L}$ is the Laplacian matrix. Then, we further need to compute the partial derivatives with respect to the previous feature maps as follows
	\begin{equation}
	\frac{\partial E}{\partial y^{l-1}_j} = \beta_j^l \sum_{m=0}^M \alpha_{i,m}^l \left[\mathcal{L}^m |_{\mathcal{N}_i^{l-1}}\right] \frac{\partial E}{\partial z^{l}_i}. \\
	\label{eq:der_fm}
	\end{equation}
	
	Our new dynamic pooling layers, as well as our statistical layer do not have parameters to be trained. Similarly to the max-pooling operator our dynamic pooling layer  permits back-propagation through the active nodes since the gradient is 0 for the non-selected nodes and not zero for the chosen ones.  
	Further, the statistical layer back-propagates the gradients as follows:
	\begin{equation}
	\frac{\partial{E}} {\partial t_{i, k}}  =\frac{1}{N} \frac{\partial{E}}{\partial \mu_{i,k}},
	\label{eq:der_stat1}
	\end{equation}
	\begin{equation}
	\frac{\partial{E}} {\partial t_{i, k}}  = \frac{2 (N-1)}{N^2} \sum_{i=1}^N \left( t_{i, k} - \mu_{i,k} \right) \frac{\partial E}{\partial \sigma^2_{i, k}},\\
	\label{eq:der_stat2}
	\end{equation}
	\noindent
	where $\mu_{i,k}, \sigma^2_{i,k}$ are the inputs to the first fully-connected layer and the outputs of the statistical layer. The derivatives $\partial E / \partial \tilde{z}_i$ are then computed as:
	\begin{equation}
	\frac{\partial E}{\partial \tilde{z}_i} = \sum_{k=0}^{K_{max}}\frac{\partial E}{\partial t_{i,k}} \frac{\partial t_{i,k}}{\partial \tilde{z}_i},
	\end{equation}
	\noindent
	where $\partial t_{i,k} / \partial \tilde{z}_i$ are simply the derivatives of Chebyshev polynomials~\cite{bb:shuman2011chebyshev} with maximum order $K_{max}$. Please note that we use the non-linear absolute function $|t_{i,k}|$ before statistical layer, therefore, the gradient at $t_{i,k}=0$ is not defined. In practice, however, we set it to $0$, which gives us a nice property of encouraging some  feature map values to be $0$ and favors sparsity.
	
	Finally, the parameters of the fully-connected layers are trained in a classical way, similarly to the training of fully-connected layers in ConvNet architectures \cite{bb:rumelhart1988learning}. 
	
	\section{Equivariance of graph filters}
	\label{s:equivaricance}
	
	In this section we analyze the equivariance property of the features produced by the graph-based convolutional layer. We start our analysis by proving that independently of the input image $y$, graph filters are equivariant to isometric transformations of $y$ such as image rotations by multiplier of $90^\circ$ and translations by an integer number of pixels in horizontal and vertical directions. We refer to these transformations as the graph isometric ones.
	
	We then show that given some constraints on the smoothness conditions on $y$, our graph-convolutional filters are \emph{quasi-equivariant} to any isometric transformations (which comprise rotations by an angle of any degree and translations by any real number of pixels). 
	
	\subsection{Graph isometric transformation}
	\label{sec:90deg}
	
	In this section we formally define isometric transformations $g$ on regular graphs and show that the Laplacian polynomial filters, introduced in Eq.~(\ref{eq:pol_filt_1}), produce inherently equivariant features with respect to such transformations $g$.
	\newtheorem{mydef}{Definition}
	\begin{mydef}
		A graph isometric transformation $g$ is a bijective mapping $g : \mathcal{V} \rightarrow \mathcal{V}$ that preserves distances between all pairs of adjacent nodes $v_i \sim v_j: w(v_i, v_j)=1$ and acts as a permutation function for signal $y$: $y_g=g(y)$ that can be formally described as
		\begin{equation}
		\small
		\forall v_k \in \mathcal{V}, \; \exists \; !\; v_j \in \mathcal{V}: y_g(v_k) = y(v_j) \;,
		\label{eq:def}
		\end{equation} 
		\noindent
		that preserves their neighbourhood and where $!$ indicates that the correspondence between vertices $v_k$ and $v_j$ is a bijective mapping. 
		\label{def:graph_eq}
	\end{mydef}
	
	%
	In other words, a graph isometric transformation $g$ is a permutation of the graph nodes $\mathcal{V}$. This results in $g$ being a combination of the following basic operations applied to the image signal $y$, defined on the regular graph $G$:
	\begin{itemize}
		\setlength\itemsep{0pt}
		\setlength{\parskip}{0pt}
		\item Image rotations by $ \frac{\pi}{2}, \pi, \frac{3\pi}{2} $;
		\item Image reflection;
		\item Image translation by an integer number of pixels.
	\end{itemize}
	If we ignore the border effects, graph isometric transformations $g$ satisfy the four fundamental group properties~\cite{bb:Herstein06} and, therefore, form a group $\mathcal{G}$.
	
	
	\newtheorem{theorem}{Theorem}
	\newtheorem{lemma}{Lemma}
	
	\begin{theorem}
		For any operation $g$ from a group of graph isometric transformations $\mathcal{G}$ the response of a polynomial filter defined in Eq.~(\ref{eq:pol_filt_1}) produces an equivariant feature representation.
		\label{t:pol_fil}
	\end{theorem}
	
	\begin{proof}[Proof of Theorem~\ref{t:pol_fil}]
		Let us denote by $y_g$ the transformed version of signal $y : y_g = g(y)$.
		Given the Def.~\ref{def:graph_eq} we can prove the theorem by induction. For simplicity and without loss of generality we prove the theorem for a  4-nn grid graph.
		
		
		\textbf{Base case: } Let us first consider the polynomial filter $\mathcal{F}$ of degree $M=1$. The result of the filtering operation $\mathcal{F}$, introduced in Eq.~(\ref{eq:pol_filt_1}), applied to a signal $y$ will then be the following:
		\begin{equation}
		\mathcal{F}(y) = (\alpha_0 + \alpha_1 \mathcal{L}) y \;, \; \{\alpha_0, \alpha_1\} \in \mathbb{R} \;.
		\label{eq:filt_1deg}
		\end{equation} 
		Here $\mathcal{L}$ is the normalized Laplacian matrix, introduced in Eq.~(\ref{eq:laplacian}).
		For the sake of simplicity and in order to avoid border effects we extend the graph by adding nodes to the image boundary and padding them with zeros. Hence all nodes $v_j$ of the initial graph have degree $d(v_j) = 4$. Consequently, the filtered signal $\mathcal{F}\big{(}y\big{)}(v_j)$ at node $v_j$ can be computed as:
		\begin{equation}
		\small
		f(v_j) \df \mF\left(y\right)(v_j) = \alpha_0 y(v_j) + \alpha_1 y(v_j) - \frac{\alpha_1}{4} \sum\limits_{i: v_i \sim v_j} y(v_i),
		\label{eq:filtered_signal}
		\end{equation} 
		where $v_i \sim v_j$ indicates adjacent nodes. 
		
		We then apply the same filter $\mathcal{F}$ to the transformed signal $y_g$, which is formed from $y$ via a graph isometric transformation $g$. According to its definition (see Eq.~(\ref{eq:def})), $g$ has the following property: $\forall v_k \mapsto v_j : y_g(v_k) = y(v_j)$ and the neighborhood of the node $v_k$ maps to the neighborhood of the node $v_j$. 
		%
		Therefore, the result of $\mathcal{F}$, applied to $y_g$ reads 
		\begin{equation}
		\small
		f_g(v_k) \df \mathcal{F}(y_g)(v_k) = \alpha_0 y_g(v_k) + \alpha_1 y_g(v_k) - \frac{\alpha_1}{4} \sum\limits_{l: v_l \sim v_k} y_g(v_l).
		\label{eq:filtered_trans_signal}
		\end{equation}
		Our goal is then to show that $f_g$ from Eq.~(\ref{eq:filtered_trans_signal}) is equivariant to $f$ from Eq.~(\ref{eq:filtered_signal}). 		
		According to Eq.~(\ref{eq:def}), the signal $y_g$ in the neighborhood of the vertex $v_k$ is identical to $y$ in the neighborhood of $v_j$. Therefore Eqs.~(\ref{eq:filtered_signal}) and~(\ref{eq:filtered_trans_signal}) are equal, which leads to 
		\begin{equation}
		f_g(v_k) = f(v_j)\;,
		\label{eq:filter_equality}
		\end{equation}
		\noindent
		where $v_k \in G$ is mapped to $v_j \in G$ via $g$.
		Further, Eq.~(\ref{eq:filter_equality}) is valid for any node $v_k$, and for each of these $v_k$'s there exists a mapping to $v_j$ according to the graph isometric transformation $g$. Therefore, by Def.~\ref{def:graph_eq}, $f_g$ and $f$ are related via $g$ as $f_g = g(f)$ and we can write:
		
		\begin{equation}
		\mathcal{F}(y_g) = g \left(\mathcal{F}(y) \right)\; \text{or, equivalently }	\mathcal{F}\left(g(y) \right) = g \left(\mathcal{F}(y) \right)\; .
		\label{eq:lastp1}
		\end{equation} 
		\noindent
		As Eq.~(\ref{eq:lastp1}) is obtained without any assumptions on the parameters $(\alpha_0, \alpha_1)$ of the polynomial filter $\mathcal{F}$, we can conclude that the response of any Laplacian polynomial filter $\mathcal{F}$ of degree $M=1$ is equivariant to graph isometric transformations $g$. The latter is, therefore, true for a special case of $\alpha_1 = 0$, which corresponds to the polynomial filters $\mathcal{F}$ of degree zero.
		
		\textbf{Inductive step: } We now show that the polynomial filter $\mathcal{F}$ of any degree is equivariant to graph isometric transformation $g$. We assume that the filter 
		\begin{equation}
		\mathcal{F}_{k-1}=\sum_{m=0}^{k-1} \alpha'_m \mathcal{L}^m
		\end{equation} 
		\noindent
		of degree $M = k-1$ is equivariant with respect to $g$. Our goal is then to show that the polynomial filter $\mathcal{F}_k$ of degree $M = k$ that has the following form:
		\begin{equation}
		\mathcal{F}_{k}(y) = \sum\limits_{m=0}^{k} \alpha_m \mathcal{L}^m y =
		\alpha_0 y + \mathcal{L} \left(\sum\limits_{m=0}^{k-1} \alpha_{m+1} \mathcal{L}^{m} y \right) \;,
		\label{eq:ind}
		\end{equation} 
		\noindent
		is also equivariant with respect to $g$.
		To do so we take a polynomial filter $\mF_{k-1}$, with coefficients $a'_m = a_{m+1}$  and rewrite Eq.~(\ref{eq:ind}) as 
		\begin{equation}
		\mathcal{F}_{k}(y) = \alpha_0 y + \mathcal{L} \mathcal{F}_{k-1}(y).
		\label{eq:ind_f}
		\end{equation} 
		\noindent
We then apply the isometric transformation $g$ to both parts of Eq.~(\ref{eq:ind_f}):
		\begin{equation}
		\small
		g(\mathcal{F}_{k}(y)) = g(\alpha_0 y + \mathcal{L} \mathcal{F}_{k-1}(y)) = \alpha_0 g(y) + \mathcal{L} g(\mathcal{F}_{k-1}(y)).
		\label{eq:ind_f_blast}
		\end{equation} 
		\noindent
As by our assumption $\mathcal{F}_{k-1}$ is equivariant with respect to $g$ and $y_g \df g(y)$, we can rewrite (\ref{eq:ind_f_blast}) as
		\begin{equation}
		g(\mathcal{F}_{k}(y)) = \alpha_0 y_g + \mathcal{L} \mathcal{F}_{k-1}(y_g) \df \mathcal{F}_{k}(y_g).
		\label{eq:ind_f_last}
		\end{equation} 
		\noindent
Based on Eqs.~(\ref{eq:ind_f_blast})~and~(\ref{eq:ind_f_last}) we can conclude that a polynomial filter of any degree produces features  that are equivariant  to the graph isometric transformations.
		
	\end{proof}
	
	
	\subsection{General isometric transformations}
	\label{s:git}
	
	In this section we extend the result of Section~\ref{sec:90deg} to general isometric transformations. We start with the introduction of our image model, we then introduce the quasi-equivariance property and prove that under some assumptions on the  image $y$ defined on the graph nodes, Laplacian polynomial filters are quasi-equivariant to general isometric transformations. The latter 
	%
	can be represented by any combination of the following basic operations:
	\begin{itemize}
		\setlength\itemsep{0pt}
		\setlength{\parskip}{0pt}
		\item Graph isometric transformations, defined in Section~\ref{sec:90deg};
		\item Image rotation by an arbitrary angle $\gamma$;
		\item Image translation by a real number of pixels $\xi$.
	\end{itemize}
	
	\subsubsection{Image model and quasi-equivariance}
	
	Camera captures 3D world and represents it as a 2D image. Therefore, an image is an approximation of a real scene with discrete pixel values. For the sake of simplicity let us represent the real 3D image signal $y$ as a two-dimensional function $f = f(a,b)$, where $0 \leq a,b \leq 1$.  
	
	Thus, the transformed version of the signal, $y_g$, is also defined $\forall (a,b)$ for any transformation $g_{\gamma,\xi}$. However, the graph filter operations are defined only for such $(a,b)$ that correspond to the graph nodes, therefore we cannot directly apply $g_{\gamma,\xi}$ to it if this condition is not satisfied. To overcome this problem we approximate $\gamma$ and $\xi$ with $\bar{\gamma}, \bar{\xi}$ that satisfy the following conditions:
	\begin{equation}
	\bar{\gamma} \; : 	
	\begin{cases}
	\bar{\gamma} = \frac{k\pi}{2},\; k \in \mathbb{Z}; \\
	||\bar{\gamma} - \gamma || \leq \frac{\pi}{4},
	\end{cases} \quad
	\bar{\xi} \; : 
	\begin{cases}
	\bar{\xi} = k,\; k \in \mathbb{Z}; \\
	||\bar{\xi} - \xi || \leq \frac{1}{2},
	\end{cases}
	\end{equation}
	\noindent
	and introduce the graph isometric transformation $\bar{g}_{\gamma, \xi} = g_{\bar{\gamma}, \bar{\xi}}$, which we refer to as \emph{closest} graph isometric transformation of the original general transformation $g_{\gamma, \xi}$.
	As discussed in Section~\ref{sec:90deg}, such transformation is defined only on graph nodes, which sample signal with steps $\Delta a$ and $\Delta b$. Finally, we introduce the quasi-equivariance property  as follows: 	
	\begin{mydef}
		The quasi-equivariance of the operator $\mF$ with respect to the isometric transformation $g_{\gamma,\xi}$ is defined as a small absolute value of the difference between signals $\mathcal{F}(g_{\gamma,\xi} (y))$ and the graph isometric transformation $\bar{g}_{\gamma,\xi}(\mathcal{F}(y))$ computed at node $v$ in $G$. Equivalently, we have:
		\begin{equation}
		\left| \mathcal{F}\big(g_{\gamma,\xi} (y)\big)(v) - \bar{g}_{\gamma,\xi} \big(\mathcal{F}(y)(v)\big) \right|  \leq \epsilon, \forall v,
		\label{eq:equivariance_ref}
		\end{equation}
		\noindent	
		where $\epsilon$ is a small value, when $\mathcal{F}$ is a quasi-equivariant.
	\end{mydef}

	
	%
	
In order to prove the quasi-equivariance of the polynomial filters $\mathcal{F}$, defined in  Eq.~(\ref{eq:pol_filt_1}), with respect to general isometric transformations we show that the response of $\mathcal{F}$ is quasi-equivariant with respect to each of the basic operations, defined in the beginning of Section~\ref{s:git}. In Section~\ref{sec:90deg} we have shown that polynomial filters are equivariant to any graph isometric transformation, which also means that the quasi-equivariance property is satisfied. Therefore, in this section we focus on proving the quasi-equivariance of $\mathcal{F}$ with respect to arbitrary image rotations and image translations by a real number of pixels.



	\subsubsection{Rotation}
	\label{sec:rot}
	
In order to prove the quasi-equivariance of polynomial filters to rotation, we first observe that for every pixel location $p$ in an image, any random image rotation around an arbitrary point can be decomposed into a set of two translations and one rotation around $p$. Therefore, in this section we focus on proving the quasi-equivariance of $\mF$ with respect to rotation around the vertex $v \in G$, where $\mF$ is applied. Then in Section~\ref{sec:trans} we discuss in more details the quasi-equivariance of $\mF$ with respect to an arbitrary image translation. Further, without loss of generality, we consider the rotation angle $\gamma$ to be in the range $[-\frac{\pi}{4},\frac{\pi}{4})$, due to the fact that any random rotation can be decomposed into an integer number of rotations by $\pi/2$ and a rotation by $\gamma \in [-\frac{\pi}{4},\frac{\pi}{4})$. As illustrated in Section~\ref{sec:90deg}, polynomial filters are equivariant with respect to rotations by $\pi/2$, thus we only need to prove their equivariance with respect to $\gamma \in [-\frac{\pi}{4},\frac{\pi}{4})$.
	
Let $y_g = g_\gamma(y)$ denote the rotated version of the signal $y$ by $\gamma \in [-\frac{\pi}{4},\frac{\pi}{4})$ around a vertex $v_r$, where a polynomial filter $\mF$ is applied. In the remainder of this section we show that under some assumptions on the second derivative of the signal $y$, the Eq.~(\ref{eq:equivariance_ref}) is valid for any polynomial filter $\mathcal{F}$ and for any $\gamma$ that defines the image rotation $g_\gamma$. Let us denote the equivariance gap by 
\begin{equation}
\Delta_\mathcal{F} \df \left| \mathcal{F}\big(g_{\gamma} (y)\big)(v) - \bar{g}_{\gamma}\big(\mathcal{F}(y)(v)\big) \right| ,
\label{eq:delta_F_quasi}
\end{equation}
where $\bar{g}_{\gamma}$ is the graph isometric transformation that is closest to $g$. Due to the fact that $\mF$ is an isotropic filter and as we are considering rotations $\gamma \in [-\frac{\pi}{4},\frac{\pi}{4})$ around vertex $v_r$, where $\mF$ is applied, we obtain the following equality $\bar{g}_{\gamma}\big(\mF(y)(v)\big) \equiv \mF(y)(v)$ and rewrite Eq.~(\ref{eq:delta_F_quasi}) as
\begin{equation}
\Delta_\mathcal{F} \df \left| \mathcal{F}\big(g_{\gamma} (y)\big)(v) - \mathcal{F}(y)(v) \right|, \forall v. 
\end{equation}
We now have the following theorem.
	
\begin{theorem}
		
		The absolute value of the difference $\Delta_\mathcal{F}$ between signals $\mathcal{F}(g_\gamma (y))$ and $\mathcal{F}(y)$ satisfies the following inequality:
		%
		\begin{equation}
		\Delta_\mathcal{F} (v) \leq \left(\sum_{k=1}^M \norm{\alpha_k} 2^{k-3} \right) \left| (1 - \sin\gamma - \cos\gamma) \bar{\mathcal{Z}} \right|,
		\label{eq:filt_diff_rot}
		\end{equation} 
		\noindent
		where $M$ is the degree of $\mathcal{F} = \sum_{k=0}^M \alpha_k \mathcal{L}^k y$ and $\bar{\mathcal{Z}}$ depends on the second derivative of the image signal $y \df f(a,b)$ as follows:
		\begin{equation}
		\small
		\bar{\mathcal{Z}} =
		\max_{0 \leq a,b \leq 1} \left| 
		\begin{bmatrix}
		\partial_a^2 f (a,b) \\
		\partial_b^2 f (a,b) \\
		\end{bmatrix}
		\right|^{\intercal}
		\begin{bmatrix}
		\Delta a^2 \\
		\Delta b^2 \\
		\end{bmatrix} + o(\Delta a^2) + o(\Delta b^2).
		\label{eq:sec_der}
		\end{equation}	
		\noindent
		Here $\Delta a, \Delta b$ are the distances between the graph nodes in horizontal and vertical directions respectively which corresponds to the resolution of the image; and $o(\cdot)$ is the ``little-o'' notation that describes function's asymptotic behavior~\cite{bb:Landau09}.
		\label{t:im_rot}
	\end{theorem}
	
	\begin{proof}[Proof of Theorem \ref{t:im_rot}]
		
		
		
		We start the proof with the observation that any polynomial filter $\mathcal{F}$ of degree $N$ is a sum of the \emph{trivial} polynomial filters $\hmF_k$ of degrees $k = [0..N]$, defined as:
			\begin{equation}
			\hmF_k \df \alpha_k \mathcal{L}^k y \;.
			\end{equation}
			\noindent
			Therefore to prove Eq.~(\ref{eq:filt_diff_rot}), we need to show that
		\begin{equation}
		\Delta_{\hmF_k} (v) \leq 2^{k-3} \norm{\alpha_k(1 - \sin\gamma - \cos\gamma) \bar{\mathcal{Z}}}, \forall k, 
		\label{eq:filt_diff_rot_triv}
		\end{equation}
		\noindent
		where $\Delta_{\hmF^k} (v)$ is the absolute difference between $\hmF^k(y_g)$ and $\hmF^k(y)$. Then $\Delta_{\mF} (v)$ can be eventually approximated using the triangle inequality as 
		\begin{equation}
		\Delta_\mF (v) \leq \sum_{k=0}^{N} \Delta_{\hmF^k} (v)
		\label{eq:filt_diff_rot_sum}
		\end{equation}
		To prove Eq.~(\ref{eq:filt_diff_rot_triv}) we use the induction method.
		\begin{figure}
			\centering
			\includegraphics[width=\linewidth]{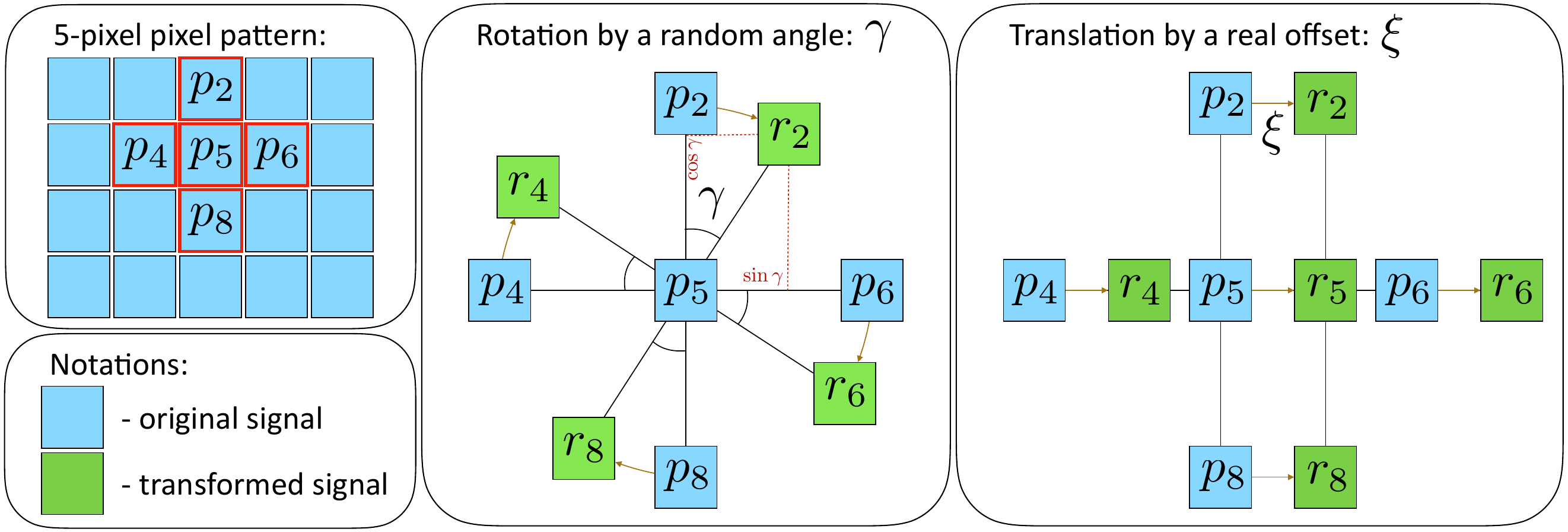} 
			\caption{[\textsc{Left}] Illustration of a 5-pixel image pattern that the polynomial filter $\mathcal{F}$ operates on, with $p_i, i \in [2,4,5,6,8]$ being the respective pixels intensities. [\textsc{Middle}] The rotation from the original image signal $y$ (blue rectangles) to the transformed image signal $y_g$ (green rectangles), after applying the transformation $g_\gamma$ around pixel $p_5$. The points $r_i, i \in [2,4,6,8]$ schematically show how the position of $p_i$ change after applying $g_\gamma$. [\textsc{Right}] The translation from the original to the transformed image signal, after applying the transformation $g_\xi$ by a real number of pixel intensities $\xi$.}
			\label{fig:pattern}
		\end{figure} 

		\textbf{Base case:} We first show that for any $\hmF$ of degree $0$ and~$1$ Eq.~(\ref{eq:filt_diff_rot_triv}) is valid. In the case of a polynomial filter of a $0$ degree, $\Delta_{\hmF^0}  = 0$, as
		\begin{equation}
		g_\gamma\left(\hmF^0 (y)\right) \df g_\gamma(\alpha_0 y) = \alpha_0 g_\gamma(y) = \hmF^0 (g_\gamma(y)).
		\label{eq:filt_diff_rot_zero_deg}
		\end{equation}
		\noindent
		Then, as depicted by Eq.~(\ref{eq:filtered_signal}), for a given node $v_r \in G$, $\hmF^1$ operates on the neighborhood of $v_r$, which can be represented as an image pattern, depicted by Fig.~\ref{fig:pattern}. For simplicity, we prove for the 4-nn graph case, however, the proof can be easily extended to the case of other regular graphs.
		For the sake of simplicity, we denote by $p_j$ the values of $y$ in the nodes of this pattern as follows:
		\begin{equation}
		p_j \df f(a_j,b_j) = y(v_j), \; \forall v_j: j \in [2,4,5,6,8] \;,
		\end{equation}
		\noindent
		where $(a_j, b_j)$ are the pixel coordinates of the node $v_j$. Assuming that $f$ is differentiable (as it is generally done for natural images) we can express $p_j$ using a Taylor approximation as
		\begin{equation}
		\small
		\begin{bmatrix}
		p_2 \\ p_6 \\ p_4 \\ p_8 \\
		\end{bmatrix} = 
		\mathcal{I} p_5 - 
		\begin{bmatrix}
		- \partial_b f(a_5, b_5) \Delta b \\
		- \partial_a f(a_5, b_5) \Delta a \\
		\partial_a f(a_5 - \Delta a, b_5) \Delta a\\
		\partial_b f(a_5, b_5- \Delta b)  \Delta b\\
		\end{bmatrix} +
		\begin{bmatrix}
		R_2(b_2) \\ R_2(a_6) \\ R_2(a_4) \\ R_2(b_8) \\
		\end{bmatrix} \;,
		\label{eq:points_o}
		\end{equation}
		\noindent
		where $\mathcal{I} = [1,1,1,1]^\intercal: \partial_a f, \partial_b f$ are the partial derivatives:
		\begin{equation}
		\partial_a f(a,b) \df \frac{\partial f(a, b)}{\partial a}\;, \quad \partial_b f(a,b) \df \frac{\partial f(a, b)}{\partial b} \;,
		\end{equation}
		and $R_2(\cdot)$ is the remainder term of the Taylor expansion
		, which reads:
		\begin{equation}
		\begin{array}{rll}
		R_2(a_k) &= o(|a_k-a_5|^2) &= o(\Delta a^2) \;, k \in [4,6] \;, \\
		R_2(b_k) &= o(|b_k-b_5|^2) &= o(\Delta b^2) \;, k \in [2,8] \;.\\
		\end{array}
		\end{equation}
		\noindent
		We further denote the new node values of the rotated pattern $g_\gamma(y)$, depicted by Fig~\ref{fig:pattern}, as follows:
		\begin{equation}
		r_j \df g_\gamma(y)(v_j),\;  j \in [2,4,6,8].
		\label{eq:rj}
		\end{equation}
		\noindent
		We can then approximate the values of $r_j$ using bilinear interpolation~\cite{bb:Szeliski10} as illustrated by Fig.~\ref{fig:pattern}. This allows us to express the values $r_j$ of each node of the rotated pattern based on the points $p_j$ of the original pattern, which results in:
		\begin{equation}
		\small
		\begin{bmatrix}
		r_2 \\ r_4 \\ r_6 \\ r_8 \\
		\end{bmatrix} = 
		\mathcal{I}  r_5 - \mathcal{A} \sin \gamma 
		- \mathcal{B} \cos \gamma 
		+ o(\Delta a^2) + o(\Delta b^2),
		%
		%
		%
		\label{eq:points_r}
		\end{equation} 
		\noindent
		where
		\begin{equation}
		\small
		\mathcal{A} = 
		\begin{bmatrix}
		- \partial_a f(a_5, b_5)\Delta a \\
		- \partial_b f(a_5, b_5)\Delta b \\
		\partial_b f (a_5, b_5-\Delta b)\Delta b \\
		\partial_a f (a_5-\Delta a, b_5)\Delta a \\
		\end{bmatrix}\;
		\mathcal{B} = 
		\begin{bmatrix}
		- \partial_b f (a_5, b_5)\Delta b \\
		\partial_a f(a_5-\Delta a, b_5)\Delta a \\
		- \partial_a f (a_5, b_5)\Delta a \\
		\partial_b f (a_5, b_5-\Delta b)\Delta b \\
		\end{bmatrix}
		\label{eq:points_r_notations}
		\end{equation}
		%
		%
		%
		
		Given the introduced notations we can compute the response of a \emph{trivial} filter $\hmF^1$ at the graph node $v_5$ as:
		\begin{equation}
		\small
		\hmF^1(y)(v_5) = - \frac{\alpha_1 }{4} \mathcal{Z}(v_5)\;,
		\label{eq:response_pat_o}
		\end{equation}

		\noindent
		where $\mathcal{Z}$ reflects the smoothness of the image signal $y$ and reads:
		\begin{equation}
		\small
		\begin{aligned}
		\mathcal{Z}(v_5) & =
		\begin{bmatrix}
		\partial_a f (a_5, b_5)  - \partial_a f (a_5-\Delta a, b_5)\\
		\partial_b f (a_5,b_5)  - \partial_b f (a_5,b_5-\Delta b)\\
		\end{bmatrix}^{\intercal}
		\begin{bmatrix}
		\Delta a \\
		\Delta b \\
		\end{bmatrix} \\
		 & + o(\Delta a^2) + o(\Delta b^2)\;.
		\end{aligned}
		\end{equation}
		\noindent
		Similarly, based on Eqs.~(\ref{eq:points_r}) and~(\ref{eq:points_r_notations}) we can compute the response of the same filter $\hmF^1$ applied to the rotated signal $y_g = g_{\gamma}(y)$ at the node $v_5$ as:
		\begin{equation}
		\hmF^{1}(y_g)(v_5) = - \frac{\alpha_1}{4} (\sin\gamma + \cos\gamma) \mathcal{Z}(v_5)
		\label{eq:response_pat_r}
		\end{equation}
		\noindent
		We can then calculate the difference between the filter responses presented in Eqs.~(\ref{eq:response_pat_o}) and~(\ref{eq:response_pat_r}) as follows:
		\begin{equation}
		\Delta_{\hmF^1}(v_5) = \frac{\alpha_1}{4} (1 - \sin\gamma - \cos\gamma) \mathcal{Z}(v_5).
		\label{eq:diff_resp}
		\end{equation} 
		%
		In order to compute $\Delta_{\hmF^1}$ in a general case (that is for an arbitrary rotation $\gamma$), we can rewrite $\mathcal{Z}$ as:
		\begin{equation}
		\mathcal{Z}(v_5) =
		\begin{bmatrix}
		\partial_a^2 f (a_5, b_5)\\
		\partial_b^2 f (a_5,b_5)\\
		\end{bmatrix}^{\intercal}
		\begin{bmatrix}
		\Delta a^2 \\
		\Delta b^2 \\
		\end{bmatrix} + o(\Delta a^2) + o(\Delta b^2),
		\label{eq:sec_der_base}
		\end{equation}
		\noindent
		which for any node $v$ in $G$ has its upper bound $\bar{\mathcal{Z}}$:
		%
		\begin{equation}
		\small
		\norm{\mathcal{Z}(v)} \leq \bar{\mathcal{Z}} = 
		\max_{0 \leq a,b \leq 1} \left|
		\begin{bmatrix}
		\partial_a^2 f (a,b)\\
		\partial_b^2 f (a,b)\\
		\end{bmatrix}
		\right|^{\intercal}
		\begin{bmatrix}
		\Delta a^2 \\
		\Delta b^2 \\
		\end{bmatrix} + o(\Delta a^2) + o(\Delta b^2) \;.
		\label{eq:upper_bound}
		\end{equation}
		\noindent
		Based on Eqs.~(\ref{eq:diff_resp}) and~(\ref{eq:upper_bound}) we obtain the following condition:
		\begin{equation}
		\Delta_{\hmF^1}(v) \leq 2^{-2}\norm{\alpha_1(1 - \sin\gamma - \cos\gamma) \bar{\mathcal{Z}}}, \forall v \in G \;,
		\label{eq:filt_rot_single_fin}
		\end{equation}
		\noindent
		for any \emph{trivial} polynomial filter $\mathcal{F}$ of degree $1$ and any image rotation $g_\gamma$, which concludes the proof of the base case.

		\textbf{Inductive step:} 
		%
		We first denote the polynomial filter of degree $M$ by $\hmF^{M} = \alpha_M \mathcal{L}^{M} y$. Then we assume that Eq.~(\ref{eq:filt_diff_rot_triv}) is satisfied for a filter $\hmF^{M-1} = \alpha_M L^{M-1} y$ of degree $M-1$ for every node $v$ in graph $G$, which brings us to the following inequality:
		\begin{equation}
		\Delta_{\hmF^{M-1}} (v) < 2^{(M-1)-3}\norm{\alpha_M (1 - \sin\gamma - \cos\gamma) \bar{\mathcal{Z}}}\;,
		\label{eq:rot_inductive_step}
		\end{equation}
		\noindent
		where $\Delta_{\hmF^{M-1}}$ is the difference between the responses of filter $\hmF^{M-1}$ applied to the original and transformed signals. For the sake of simplicity we denote the right hand side of Eq.~(\ref{eq:rot_inductive_step}) by $\epsilon_\gamma$. Then our goal is to prove that 
		\begin{equation}
		\norm{\Delta_{\hmF^M} (v)} \leq 2 \epsilon_\gamma, \quad \forall v \in G\;.
		\end{equation}
		To do so, we use the same representation of the filter $\hmF^M$, as in Eq.~(\ref{eq:ind_f}), which allows us to rewrite $\Delta_{\hmF^M}$ as:
		\begin{equation}
		\Delta_{\hmF^{M}} = \mathcal{L} \hmF^{M-1}(y_g) - g_\gamma(\mathcal{L}\hmF^{M-1}(y)) \;,
		\end{equation}
		\noindent
		where $\mathcal{L}$ is the Laplacian matrix. Further, due to the linearity of the isometric transformation $g_\gamma$, we can rewrite $\Delta_{\hmF^{M}}$ as
		\begin{equation}
		\small
		\Delta_{\hmF^{M}} = \mathcal{L} \left(\hmF^{M-1}(y_g) - g_\gamma( \hmF^{M-1}(y))\right) \df \mathcal{L} \left(\Delta_{\hmF^{M-1}}\right)\;,
		\label{eq:ind_rot}
		\end{equation}
		\noindent
		that can be expressed as
		%
		\begin{equation}
		\small
		\Delta_{\hmF^{M}} (v_j) = \Delta_{\hmF^{M-1}} (v_j) + \frac{1}{4} \sum\limits_{i: v_i \sim v_j} \Delta_{\hmF^{M-1}} (v_i), \; \forall v_j \in G.
		\label{eq:rot_fin_diff}
		\end{equation}
		\noindent 
			We can then apply triangle inequality to Eq.~(\ref{eq:rot_fin_diff}) and write:
		\begin{equation}
		\small
		\begin{aligned}
		\Delta_{\hmF^{M}} (v) & \leq \max_{v \in G} \Delta_{\hmF^{M}} (v) \leq \\
		& \leq \max_{v \in G} \left( \Delta_{\hmF^{M-1}}(v) + \frac{1}{4} \sum\limits_{i: v_i \sim v} \Delta_{\hmF^{M-1}} (v_i) \right)
		\end{aligned}
		\label{eq:rot_fin_diff_int}
		\end{equation}
		\noindent
		As every node $v$ has at most $4$ neighbors and given Eq.~({\ref{eq:rot_inductive_step}): $| \Delta_{\hmF^{M-1}} (v) | \leq \epsilon_\gamma, \forall v \in G$, we can rewrite Eq.~(\ref{eq:rot_fin_diff_int}) as}:
		\begin{equation}
		\Delta_{\hmF^{M}}(v) \leq \max_{v \in G} \Delta_{\hmF^{M}} (v) = \epsilon_\gamma + \frac{1}{4}(4 \epsilon_\gamma) = 2\epsilon_\gamma \;.
		\label{eq:ind_rot_fin}
		\end{equation}
		\noindent
		Based on Eqs.~(\ref{eq:rot_inductive_step})~and~(\ref{eq:ind_rot_fin}), we can then write the general condition on $\Delta_{\hmF^{M}}(v)$ for a \emph{trivial} polynomial filter of any degree $M$, for every node $v$ in $G$ as follows:
		\begin{equation}
		\Delta_{\hmF^{M}} (v) \leq 2^{M-3} \norm{\alpha_{M} (1 - \cos \gamma - \sin \gamma) \bar{\mathcal{Z}}} \;, 
		\label{eq:rot_fin}
		\end{equation}
		\noindent
		which is identical to Eq.~(\ref{eq:filt_diff_rot_triv}). We then obtain Eq.~(\ref{eq:filt_diff_rot}) by simply combining Eq.~(\ref{eq:rot_fin}) with Eq.~(\ref{eq:filt_diff_rot_sum}), which concludes the proof of the theorem.
		%
	\end{proof}
	
	It is worth noting that in the special case of $\gamma$ being equal to one of the following values: $[0, \pi/2, \pi, 3\pi/2]$, the general isometric transformation $g_\gamma$ becomes a graph isometric transformation, which is thoroughly discussed in Section~\ref{sec:90deg}. This essentially means that for any signal $y$ the difference between filter responses is $0$. As we can see from Eq.~(\ref{eq:filt_diff_rot}), this is indeed true, as $\Delta_\mathcal{F} = 0$ independently of $\bar{\mathcal{Z}}$. 
	
	\subsubsection{Translation}
	\label{sec:trans}
	
	We now prove the quasi-equivariance of $\mathcal{F}$ with respect to an arbitrary image translation $g_\xi$. We start with a simple observation that any arbitrary translation of an image signal by a real number of pixels can be decomposed into a graph isometric one (i.e translation by an integer number of pixels) and an image translation by $\xi$ that is less than $1$ pixel. As illustrated in Section~\ref{sec:90deg}, polynomial filters $\mathcal{F}$ are equivariant with respect to the graph isometric translation, therefore, in this section we prove their quasi-equivariance with respect to translation by $\xi \in [-\frac{1}{2},\frac{1}{2})$ pixels.


	 Similarly to Section~\ref{sec:rot}, let $y_g = g_\xi(y)$ denote the translated version of the signal $y$ by $\xi \in [-\frac{1}{2},\frac{1}{2})$. Let us then denote the equivariance gap by 
	 \begin{equation}
	 \Delta_\mathcal{F} \df \left| \mathcal{F}\big(g_{\xi} (y)\big)(v) - \bar{g}_{\xi}\big(\mathcal{F}(y)(v)\big) \right| ,
	 \label{eq:delta_F_quasi_tr}
	 \end{equation}
	 where $\bar{g}_{\xi}$ is graph isometric transformation that is the closest to $g$. Due to the fact that $\mF$ is an isotropic filter and as we are considering rotation $\gamma \in [-\frac{\pi}{4}, \frac{\pi}{4}]$ around vertex $v_r$, where $\mF$ is applied], we obtain the following equality $\bar{g}_{\xi}\big(\mF(y)(v)\big) \equiv \mF(y)(v)$ and rewrite Eq.~(\ref{eq:delta_F_quasi_tr}) as
	 \begin{equation}
	 \Delta_\mathcal{F} \df \left| \mathcal{F}\big(g_{\xi} (y)\big)(v) - \mathcal{F}(y)(v) \right|, \forall v.
	 \label{eq:equivariance_ref_trans_notations_smpl}
	 \end{equation} 
	\noindent
	and prove the following theorem.
	
	
	
	\begin{theorem}
		The absolute value of the difference $\Delta_\mathcal{F}(v)$ between signals $\mathcal{F}(g_\xi (y))$ and $\mathcal{F}(y)$ for any node $v$ defined on the graph $G$, satisfies the following inequality:
		\begin{equation}
		\Delta_\mathcal{F} (v) \leq \left( \sum_{k=1}^{M} \norm{a_k} 2^{k-3} \right) \norm{\bar{\mathcal{Z}} + o(\Delta a^2) + o(\Delta b^2)}\;,
		\label{eq:filt_diff}
		\end{equation} 
		\noindent
		where $M$ is the degree of $\mathcal{F} = \sum_{k=0}^M \alpha_k \mathcal{L}^k y$ and $\bar{\mathcal{Z}}$ depends on the third partial derivatives of the image signal $y \df f(a,b)$ as
		\begin{equation}
		\bar{\mathcal{Z}} = 
		\max_{0 \leq a,b \leq 1} \left|
		\begin{bmatrix}
		\partial^3_a f(a,b) \\
		\partial^3_b f(a,b) \\
		\partial_a\partial^2_b f(a,b) \\
		\partial_b\partial^2_a f(a,b) \\
		\end{bmatrix}
		\right|^{\intercal}
		\begin{bmatrix}
		\Delta a^3 \\
		\Delta b^3 \\
		\Delta b^2 \Delta a \\
		\Delta a^2 \Delta b \\
		\end{bmatrix} \\
		\label{eq:trans_z}
		\end{equation}
		\noindent
		with $\Delta a, \Delta b$ being the distances between the graph nodes in horizontal and vertical directions of the 2D image, respectively.
		\label{t:im_trans}
	\end{theorem}
	
	\begin{proof}[Proof of Theorem~\ref{t:im_trans}]
		We follow similar steps as in the proof of Theorem~\ref{t:im_rot} and show that, for each \emph{trivial} filter $\hmF^k = \alpha_k \mathcal{L}^k y$, the following relation is valid:
		\begin{equation}
		\Delta_{\hmF^k} = 2^{k-3} \norm{\alpha_k (\bar{\mathcal{Z}} + o(\Delta a^2) + o(\Delta b^2))} \;.
		\label{eq:trans_triv_diff}
		\end{equation}
		\noindent
		Here $\Delta_{\hmF^k}$ is the difference between responses of $\hmF^k = \alpha_k \mathcal{L}^k y$ applied to the original signal $y$ and the transformed version $y_g$, and $\bar{\mathcal{Z}}$ is defined in Eq.~(\ref{eq:trans_z}). Then, based on Eq.~(\ref{eq:trans_triv_diff}) and triangle inequality, we can show that Eq.~(\ref{eq:filt_diff}) is valid for any polynomial filter $\mF$  that can be written as a sum of trivial filters.
		
		Similarly to Theorem~\ref{t:im_rot}, in order to prove Eq.~(\ref{eq:trans_triv_diff}), we use the induction method. Further, we use the fact that any random translation $g_\xi$ can be represented as a sum of the translations in horizontal and vertical directions.
		
		\textbf{Base case:} We need to prove that Eq.~(\ref{eq:trans_triv_diff}) is valid for any filter $\hmF$ of degree $0$ and $1$. The proof of the following equality: 			
		\begin{equation}
		g_\gamma\left(\hmF^0 (y)\right) \df g_\gamma(\alpha_0 y) = \alpha_0 g_\gamma(y) = \hmF^0 (g_\gamma(y)),
		\label{eq:filt_diff_rot_zero_deg_shift}
		\end{equation}
		is identical to the proof of Theorem~\ref{t:im_rot}, which validates the base case for the \emph{trivial} filter of degree 0. Then we need to show that Eq.~(\ref{eq:trans_triv_diff}) is valid for any \emph{trivial} filter $\hmF^{1}$ of degree $1$. To do so, we first compute the response of $\hmF^1(y) = \alpha_1 \mathcal{L} y$ at a random vertex $v_5 \in G$ with the coordinates $(a_5,b_5)$ (see Fig.~\ref{fig:pattern} (right)) as:
		\begin{equation}
		\hmF^1(y)(v_5) = \alpha_1\left(f(a_5,b_5) - \frac{1}{4}\sum_{i=[2,4,6,8]}f(a_i,b_i)\right) \;.
		\label{eq:response_mb}
		\end{equation}
		\noindent
		We then apply the same $\hmF^1$ to the transformed signal $y_{g}$ that is a version of the signal $y$, shifted horizontally  by $\xi$ pixels. Using a Taylor expansion, we obtain the following:
		\begin{equation}
		\small
		\begin{aligned}
		\hmF^1(y_g)(v_5) & = \alpha_1\left(f(a_5,b_5) - \frac{1}{4}\sum_{i=[2,4,6,8]}f(a_i,b_i)  + o(\Delta a^2)\right) + \\
		& + \frac{\alpha_1}{4}
		\begin{bmatrix}
		\partial_a f(a_5,b_5) \\ 
		\partial_a f(a_5-\Delta a, b_5) \\ 
		\partial_a f(a_5+\Delta a, b_5) \\ 
		\partial_a f(a_5, b_5+\Delta b) \\ 
		\partial_a f(a_5, b_5-\Delta b) \\ 
		\end{bmatrix}^\intercal
		\begin{bmatrix}
		4 \\ -1 \\ -1 \\ -1 \\ -1 \\
		\end{bmatrix} \Delta a \;.
		\end{aligned}
		\label{eq:response_mafter}
		\end{equation}
		\noindent
		Equipped with Eqs.~(\ref{eq:response_mb})~and~(\ref{eq:response_mafter}) we compute the difference between filter responses $\Delta_{\hmF^1}$
		at node $v_5$ in $G$ and obtain
		\begin{equation}
		\small
		\Delta_{\hmF^1}(v_5) = \frac{\alpha_1}{4}
		\begin{bmatrix}
		\partial_a f(a_5,b_5) \\ 
		\partial_a f(a_5-\Delta a, b_5) \\ 
		\partial_a f(a_5+\Delta a, b_5) \\ 
		\partial_a f(a_5, b_5+\Delta b) \\ 
		\partial_a f(a_5, b_5-\Delta b) \\ 
		\end{bmatrix}^\intercal
		\begin{bmatrix}
		4 \\ -1 \\ -1 \\ -1 \\ -1 \\
		\end{bmatrix} \Delta a + o(\Delta a^2) \;.
		\label{eq:trans_diff}
		\end{equation}
		After applying matrix multiplication and Taylor expansion, Eq.~(\ref{eq:trans_diff}) boils down to
		\begin{equation}
		\small
		\Delta_{\hmF^1}(v_5) = \frac{\alpha_1}{4}
		\begin{bmatrix}
		\partial^3_a f(a_5,b_5) \Delta a\\
		\partial_a\partial^2_b f(a_5,b_5) \Delta a\\
		\end{bmatrix}^\intercal
		\begin{bmatrix}
		\Delta a^2 \\
		\Delta b^2 \\
		\end{bmatrix} + o(\Delta a^2) + o(\Delta b^2)\;,\\
		\label{eq:horizontal_trans}
		\end{equation}
		\noindent
		which permits to write the following upper bound for all nodes $v$ in the graph $G$:
		\begin{equation}
		\small
		\begin{aligned}
		\Delta_{\hmF^1} (v) & \leq \max_{v \in G} \Delta_\mathcal{F} (v)  = o(\Delta a^2) + o(\Delta b^2) + \\
		& + 
		\max_{0 \leq a,b \leq 1} 
		\left| \frac{\alpha_1}{4}
		\begin{bmatrix}
		\partial^3_a f(a,b)\Delta a\\
		\partial_a\partial^2_b f(a,b) \Delta a\\
		\end{bmatrix}
		\right|^\intercal
		\begin{bmatrix}
		\Delta a^2 \\
		\Delta b^2 \\
		\end{bmatrix}. \\
		\end{aligned}
		\label{eq:response_del_trans_res}
		\end{equation}
		
		This concludes the derivation of the upper bound on $\left|\Delta_{\hmF^1} (v)\right|$ for an arbitrary horizontal translation $g_\xi$. We then perform the exact same steps for the vertical image translation and obtain:
		\begin{equation}
		\small
		\begin{aligned}
		\Delta_{\hmF^1} (v) & \leq \max_{v \in G} \Delta_\mathcal{F} (v)  = o(\Delta a^2) + o(\Delta b^2) + \\
		& + 
		\max_{0 \leq a,b \leq 1} 
		\left| \frac{\alpha_1}{4}
		\begin{bmatrix}
		\partial^3_b f(a,b)\Delta b\\
		\partial_b\partial^2_a f(a,b) \Delta b\\
		\end{bmatrix}
		\right|^\intercal
		\begin{bmatrix}
		\Delta b^2 \\
		\Delta a^2 \\
		\end{bmatrix}\;.\\
		\end{aligned}
		\label{eq:response_del_vtrans_res}
		\end{equation}
		\noindent
		Finally, for an arbitrary image translation we add  Eqs.~(\ref{eq:response_del_trans_res})~and (\ref{eq:response_del_vtrans_res}) and obtain Eq.~(\ref{eq:trans_triv_diff}), which proves the base case of the theorem.
		
		\textbf{Inductive step:} 
		To prove the inductive step of the theorem we need to show that for any filter $\hmF^M$ of degree $M$ and every node $v$ in $G$ the following holds:
			\begin{equation}
			\Delta_{\hmF^M} (v) \leq 2\epsilon_\xi\;,
			\end{equation}
			\noindent
			under the assumption that for a filter $\hmF^{M-1} = \alpha_M \mathcal{L}^{M-1} y$ of degree $M-1$, the following is valid:
			\begin{equation}
			\small
			\Delta_{\hmF^{M-1}} (v) \leq \epsilon_\xi, \; \epsilon_\xi = 2^{k-3} \norm{\alpha_M (\bar{\mathcal{Z}} + o(\Delta a^2) + o(\Delta b^2))}.
			\end{equation}
			\noindent
			The proof of this fact is identical to the inductive step of the Theorem~\ref{t:im_rot}, which concludes the proof of this theorem.
		
		%
	\end{proof}

	As we can see from Theorem~{\ref{t:im_trans}} polynomial filter quasi-equivariant to arbitrary image translation. Further, Eqs.~(\ref{eq:response_del_trans_res}) and (\ref{eq:response_del_vtrans_res}) show that the equivariance gap depends on the image properties, such as the maximum value of the third derivative of the input signal $y$, which is directly related to the smoothness of $y$. 
		
	\subsubsection{Discussion}
	\label{s:discussion}
	
		The Theorems~\ref{t:im_rot} and~\ref{t:im_trans} permit to obtain the following result.
	\begin{lemma}
		For a high resolution image signals $y$ and $y_g$ (i.e., the distances between graph nodes $\Delta a, \Delta b$ are small) the difference between the responses of the polynomial filters $\Delta_\mathcal{F}$, defined in Eqs.~(\ref{eq:filt_diff_rot}) and~(\ref{eq:filt_diff}) is smaller than (or equal to) the difference of the responses of the same filter $\mathcal{F}$, applied to the lower resolution versions of the same images $y$ and $y_g$. 
	\end{lemma}
	The proof of this naturally follows from the Theorems~\ref{t:im_rot} and~\ref{t:im_trans}. Based on these theorems, we can derive a formal condition on the resolution of the image signal $y$ that guarantees that the response of any polynomial filter $\mathcal{F}$ of a given degree is quasi-equivariant to random image rotations. 
	
	\begin{lemma}
		For a polynomial filter $\mathcal{F}$ of a given degree there exist positive values $\Delta a$ and $\Delta b$, which define the resolution of an image signal $y=f(a,b)$, such that the equivariance gap $\Delta_\mathcal{F}$ between $\mathcal{F}(y)$ and its rotated version $\mathcal{F}(y_g)$ is lower then a predefined constant $\epsilon_\text{def}$.
		\label{l:im_res}
	\end{lemma}
	\begin{proof}[Proof of Lemma \ref{l:im_res}]
		Our goal is to derive an upper bound on the values of $\Delta a$ and $\Delta b$ such that $\Delta_\mathcal{F}$ is lower then a predefined constant $\epsilon_\text{def} > 0$. The values $\Delta a$ and $\Delta b$ are directly related to the resolution of the image $y$ as they define the distance between the neighboring nodes of the graph $G$, which is small for high resolution images and large for the low resolution ones. Here we show the derivation of the upper bounds on $\Delta a$ and $\Delta b$ for a polynomial filter $\mathcal{F}$ of degree $1$. Similar proof can be done for the polynomial filters of any arbitrary degree.
		
		Let us introduce the following notations:
		\begin{equation}
		\begin{aligned}
		\epsilon_a \df \left(\partial^2_a f (a_5,b_5) \Delta a^2 + o(\Delta a^2)\right)(1-\sin\gamma -\cos\gamma), \\
		\epsilon_b \df \left(\partial^2_b f (a_5,b_5) \Delta b^2 + o(\Delta b^2)\right)(1-\sin\gamma -\cos\gamma). \\ 
		\end{aligned}
		\end{equation}
		\noindent
		Based on Eqs.~(\ref{eq:diff_resp})~and~(\ref{eq:sec_der}) we can represent the equivariance gap $\Delta_\mathcal{F}$ as 
		\begin{equation}
		\Delta_\mathcal{F} \df |\epsilon_a + \epsilon_b| \leq |\epsilon_a| + |\epsilon_b| \leq  \frac{\epsilon_\text{def}}{2} + \frac{\epsilon_\text{def}}{2} = \epsilon_\text{def} \;.
		\label{eq:cond_diff}
		\end{equation}
		\noindent
		Then, we can express the bounds on the distances between the adjacent grid nodes in horizontal and vertical directions as
		\begin{equation}
		\begin{aligned}
		\Delta a \leq \sqrt{\frac{\epsilon_\text{def}}{2} \left|\left(\partial^2_a f + o(1)\right)(1-\sin\gamma -\cos\gamma)\right|^{-1}}\;, \\
		\Delta b \leq \sqrt{\frac{\epsilon_\text{def}}{2} \left|\left(\partial^2_b f + o(1)\right)(1-\sin\gamma -\cos\gamma)\right|^{-1}}\;. \\
		\end{aligned}
		\label{eq:response_pattern_r}
		\end{equation} 
		
		Therefore, for the image resolution defined by $\Delta a$, $\Delta b$ from Eq.~(\ref{eq:response_pattern_r}) the equivariance gap $\Delta \mathcal{F}$ is lower or equal than $\epsilon_\text{def}$.	
	\end{proof}
	
	
	To sum up, we have shown that our polynomial filters are quasi-equivariant to random image rotations and translations. Additionally, as shown in Section~\ref{sec:90deg} our polynomial filter $\mathcal{F}$ is equivariant to reflection transformation, which is an even stronger property than quasi-equivariance. This, altogether, proves that our polynomial filters are quasi-equivariant to any isometric transformation as the latter can be represented as the combination of rotation, translation and reflection.
	
	
	\section{Experiments}
	\label{s:exp}
	
	In this section we analyze the results and compare our network to the state-of-the-art transformation-invariant classification algorithms. We first describe the experimental settings. We then analyze our architecture and the influence of the different design parameters. Further, we show that equivariance gap $\Delta_\mF$ is lower for higher resolution images which confirms our theoretical results. Finally we compare our network to the state-of-the-art transformation-invariant classification algorithms.

	\begin{table*}[!ht]
		\centering
		\begin{tabularx}{\linewidth}{Xl}
			\toprule
			Method & Architecture \\
			\midrule
			{\bf Experiments on MNIST-012} & \\
			$\quad$ ConvNet~\cite{bb:lecun} & C[3]-P[2]-C[6]-P[2]-FC[50]-FC[30]-FC[10] \\
			$\quad$ STN~\cite{bb:STN} & C[3]-ST[6]-C[6]-ST[6]-FC[50]-FC[30]-FC[10] \\
			$\quad$ TIGraNet & SC[3, 3]-DP[300]-SC[6, 3]-DP[100]-S[10]-FC[50]-FC[30]-FC[10] \\
			\midrule
			{\bf Other experiments} & \\
			$\quad$ ConvNet~\cite{bb:lecun} & C[10]-P[2]-C[20]-P[2]-FC[500]-FC[300]-FC[100] \\
			$\quad$ STN~\cite{bb:STN} & C[10]-ST[6]-C[20]-ST[6]-FC[500]-FC[300]-FC[100] \\
			$\quad$ DeepScat~\cite{bb:oyallon2015deep} & W[2, 5]-PCA[20] \\
			$\quad$ HarmNet~\cite{bb:harm}  & HRC[1, 10]-HCN[10]-HRC[10, 10]-HRC[10, 20]-HCN[20]-HRC[20, 20] \\
			$\quad$ TIGraNet & SC[10, 4]-DP[600]-SC[20, 4]-DP[300]-S[12]-FC[500]-FC[300]-FC[100] \\
			\bottomrule 
			\\
		\end{tabularx}
		\caption{Architectures used for the experiments. 
		}
		\label{tab:arch_part2}
	\end{table*}
	
	\subsection{Experimental settings}
	\label{s:init}
	
	The initialization of the system may have some influence on the actual values of the parameters after training. We have chosen to initialize the parameters $\alpha_{i,m}^l$~(Eq. \ref{eq:pol_filt_1}) of our spectral convolutional filters so that the different filters uniformly cover the full spectral domain. We first create a set of $Z$ overlapping rectangular functions $z(\mu, a_i, b_i)$
	
	\begin{equation}
	z(\mu, a_i, b_i) = 
	\begin{cases} 
	1 & \mbox{if } a_i < \mu < b_i, \\ 
	0 & \mbox{otherwise.}
	\end{cases}
	\end{equation}
	
	The non-zero regions for all functions have the same size, and the set of functions covers the full spectrum of the normalized laplacian $\mathcal{L}$, i.e., $[0, 2]$. We finally approximate each of these rectangular functions by a $M$-order polynomial, which produces a set of initial coefficients $\alpha_{i,m}^l$ that are used to define the initial version of the spectral filter $\mathcal{F}_i^{l}$. Then, the initial values of the parameters $\beta$ in the spectral convolutional layer are distributed uniformly in $[0, 1]$ and those of the parameters in the fully-connected layers are selected uniformly in $[-1, 1]$. 
	
	We run experiments with different numbers of layers and parameters. For each architecture, the network is trained using back-propagation with Adam~\cite{bb:adam} optimization. The exact formulas of the partial derivatives are provided in the supplementary material. 
	
	Our architecture has been trained and tested on different datasets, namely:
	\begin{itemize}[topsep=0pt, partopsep=0pt]
		\setlength\itemsep{0pt}
		\setlength{\parskip}{1pt}
		\item \textbf{MNIST-012.} 
		This is a small subset of the MNIST dataset \cite{bb:lecun-mnisthandwrittendigit-2010}. It includes 500 training, 100 validation and 100 test images selected randomly from the MNIST images
		with labels `0', `1' and `2'. This small dataset permits studying the behavior of our network in detail and to analyze the influence of each of the layers on the performance.
		\item
		\textbf{Rotated and translated MNIST.} 
		To test the invariance to rotation and translation of the objects in an image
		we create MNIST-rot and MNIST-trans datasets respectively.
		Both of these datasets contain 50k training, 3k validation and $\sim$9k test images. We use all MNIST digits \cite{bb:lecun-mnisthandwrittendigit-2010} except `9' as it is rotated version resembles `6'.
		In order to be able to apply transformation to the digits, we resize the MNIST-rot to the size $26 \times 26$ and MNIST-trans to the $34 \times 34$.
		The training and validation data of these datasets contain images of digits without any transformation. However, the testing set of  MNIST-rot contains randomly rotated digits by angles in range $[0\degree, 360\degree]$, while the testing set of MNIST-trans comprises randomly translated MNIST examples up to $\pm 6$ pixels in both vertical and horizontal directions.
		\item
		\textbf{ETH-80.} 
		This dataset~\cite{bb:ETH80} contains images of $80$ objects that belong to $8$ classes. Each object is represented by $41$ images captured from different viewpoints located on a hemisphere (see Fig.~\ref{fig:eth80}). The dataset shows a real life example where isometric transformation invariant features are useful for the object classification. We resize the images to $[50 \times 50]$ and randomly select $2300$ and $300$ of them as the training and validation sets and we use the rest of the images for testing.
	\end{itemize}
	\noindent
	For all these datasets, we define $G$ as a grid graph where each node corresponds to a pixel location and is connected with 8 its nearest neighbors with a weight that is equal to $1$. The pixel luminance values finally define the signal $y$ on the graph $G$ for each image.

	\begin{figure}
		\centering
		\begin{tabular}{cccccc}
			\vspace{-0.3cm} 	
			\hspace{-0.5cm} 	
			\includegraphics[width=0.22\linewidth]{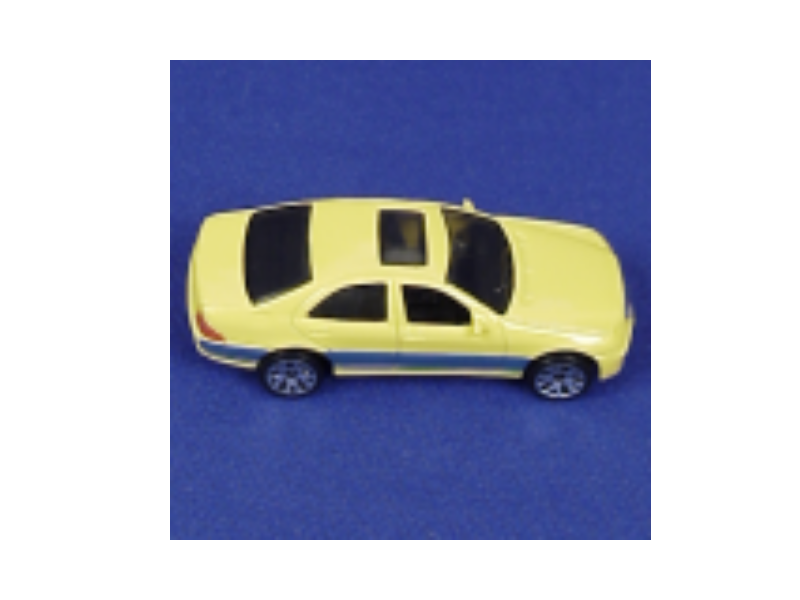} &
			\hspace{-1cm} 	
			\includegraphics[width=0.22\linewidth]{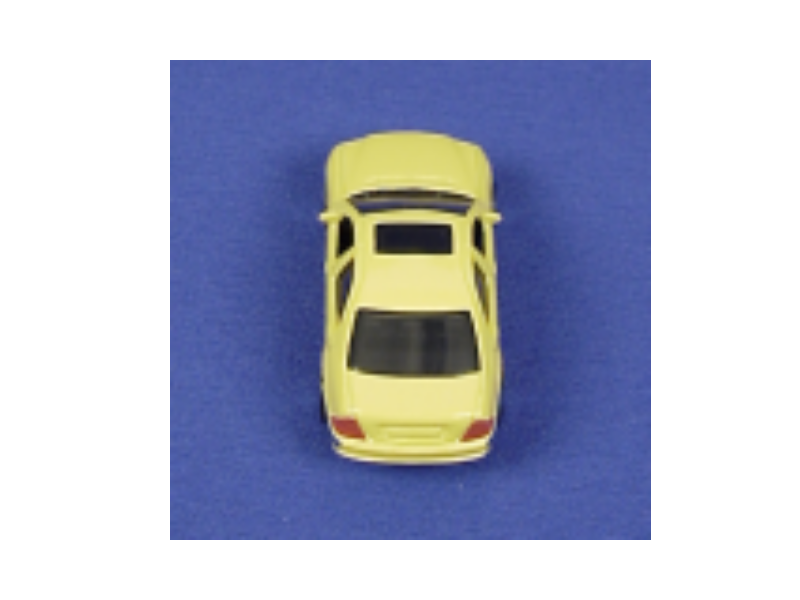} &
			\hspace{-1cm} 
			\includegraphics[width=0.22\linewidth]{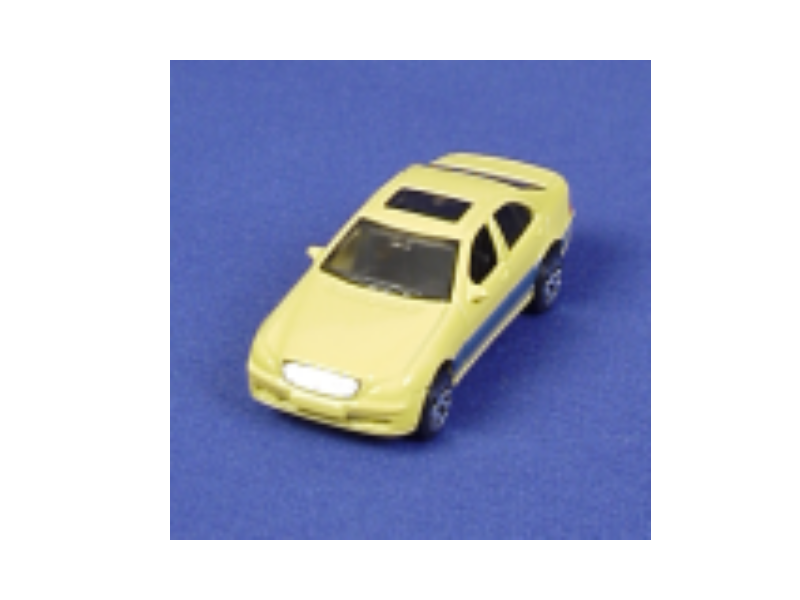} &
			\hspace{-1cm} 
			\includegraphics[width=0.22\linewidth]{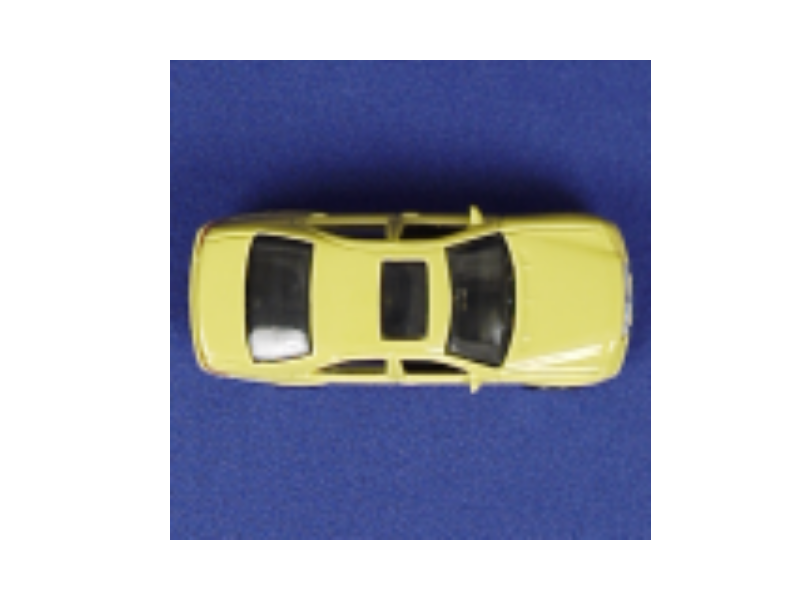} &
			\hspace{-1cm} 	
			\includegraphics[width=0.22\linewidth]{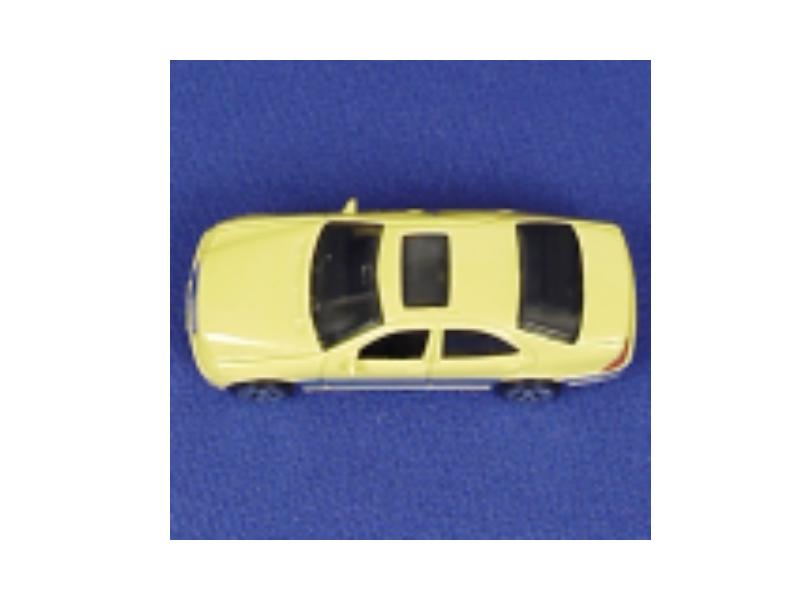} &
			\hspace{-1cm} 
			\includegraphics[width=0.22\linewidth]{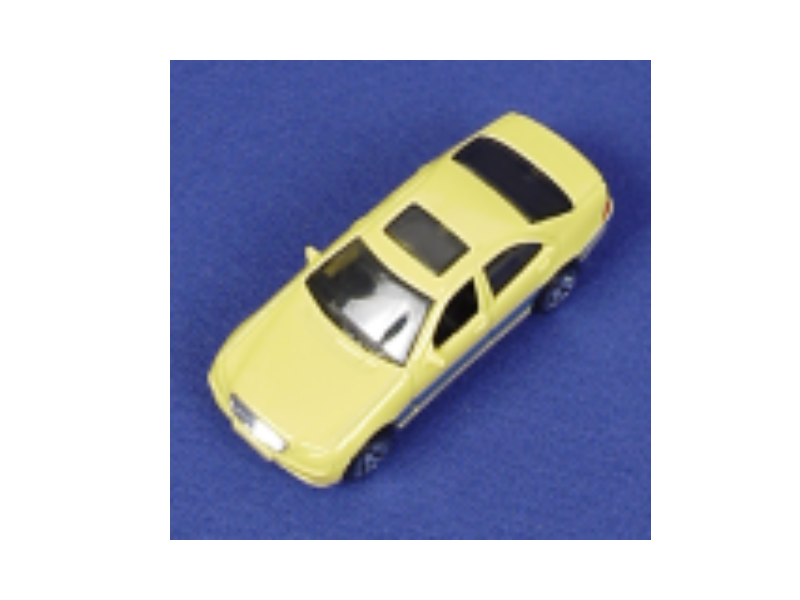} \\
			\vspace{-0.3cm} 	
			\hspace{-0.5cm} 	
			\includegraphics[width=0.22\linewidth]{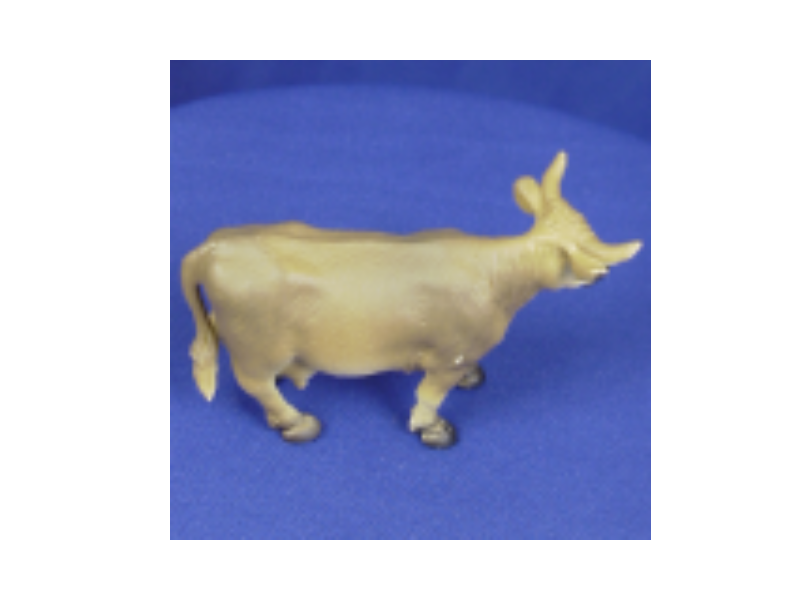} &
			\hspace{-1.0cm} 	
			\includegraphics[width=0.22\linewidth]{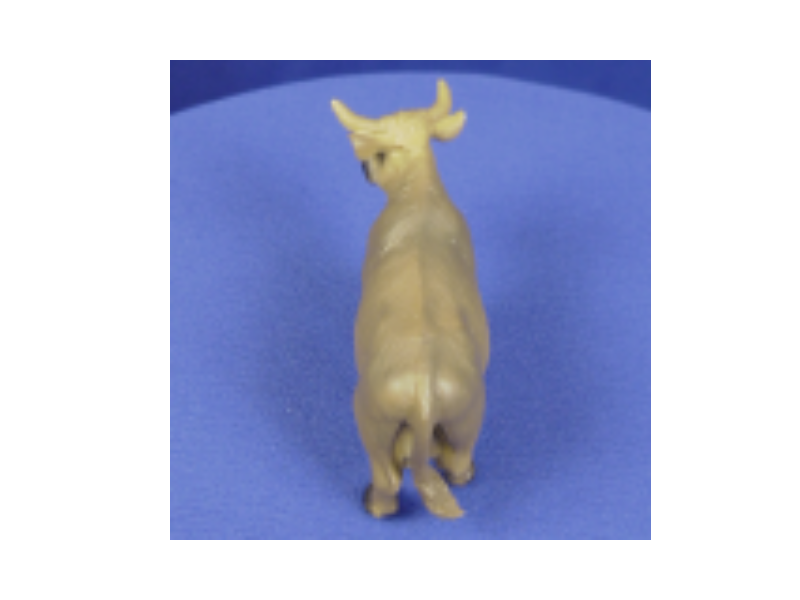} &
			\hspace{-1.0cm} 		
			\includegraphics[width=0.22\linewidth]{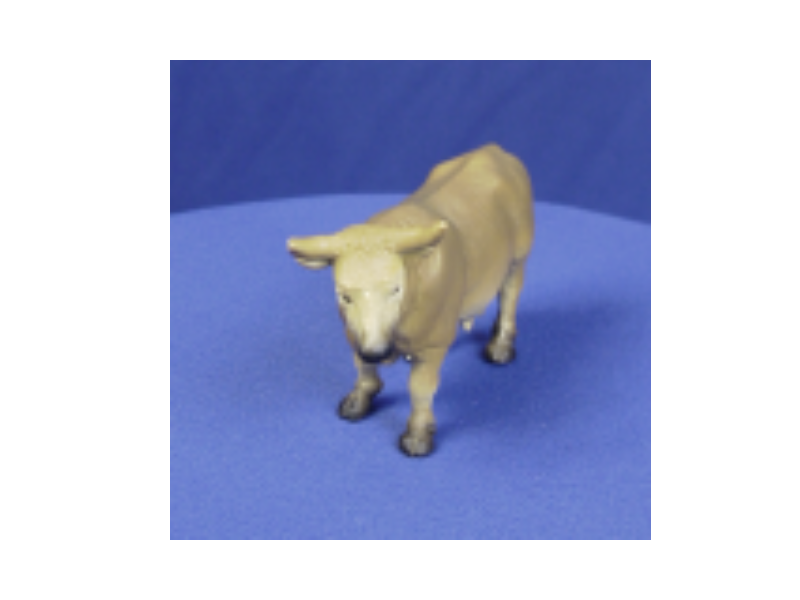} &
			\hspace{-1.0cm} 
			\includegraphics[width=0.22\linewidth]{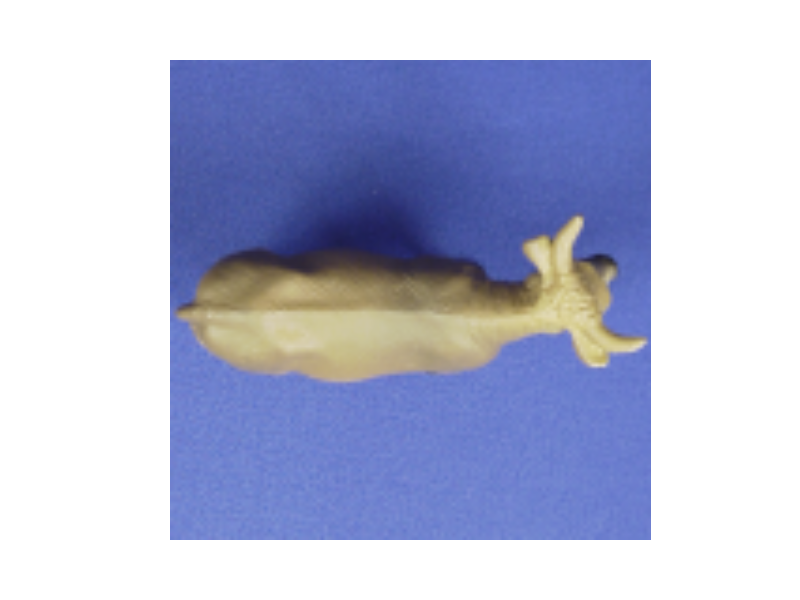} &
			\hspace{-1.0cm} 
			\includegraphics[width=0.22\linewidth]{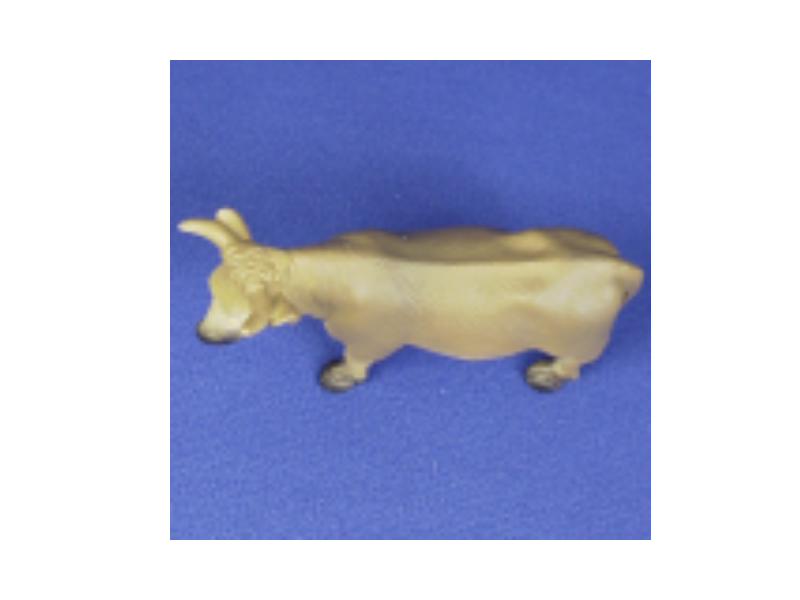} &
			\hspace{-1.0cm} 
			\includegraphics[width=0.22\linewidth]{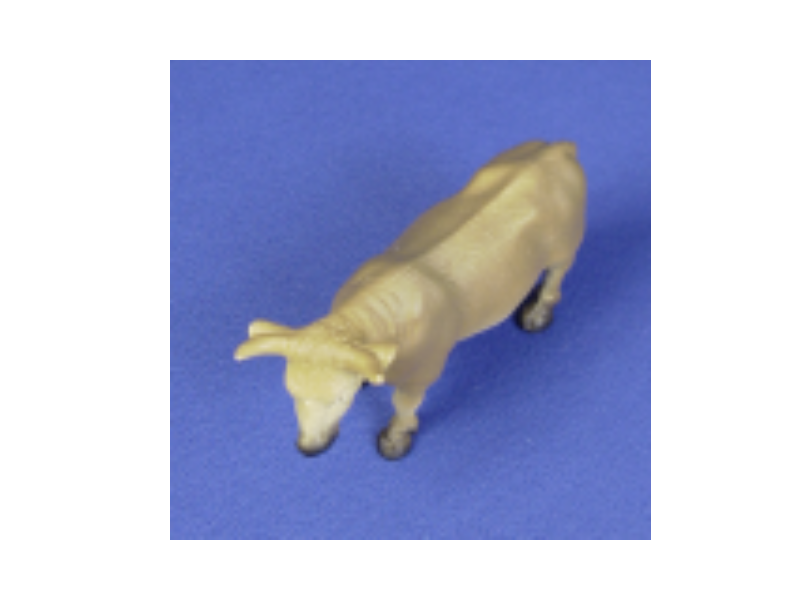} \\
			\vspace{-0.3cm} 	
			\hspace{-0.5cm} 	
			\includegraphics[width=0.22\linewidth]{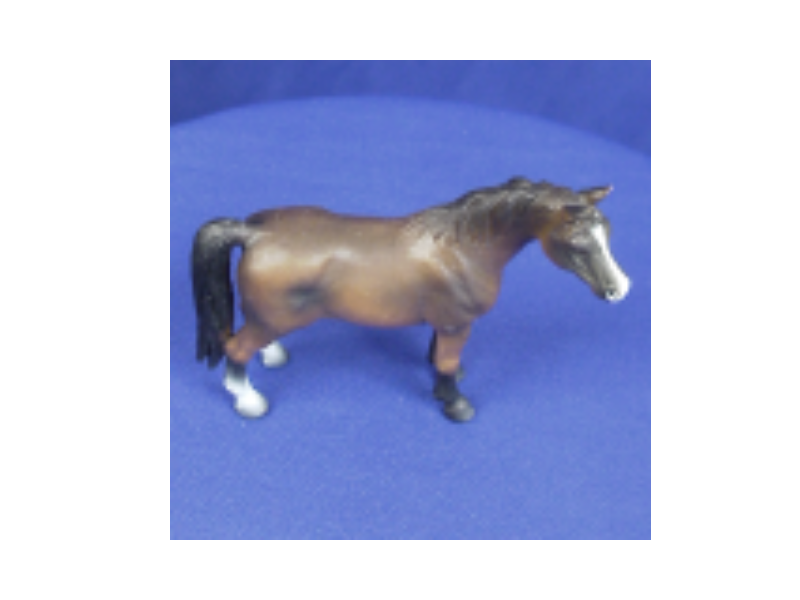}  &
			\hspace{-1.0cm} 	
			\includegraphics[width=0.22\linewidth]{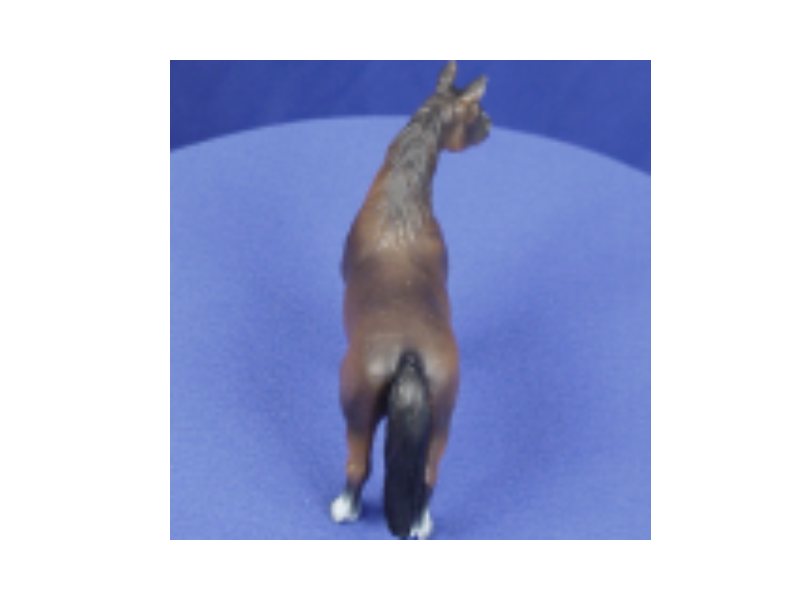} &
			\hspace{-1.0cm} 		
			\includegraphics[width=0.22\linewidth]{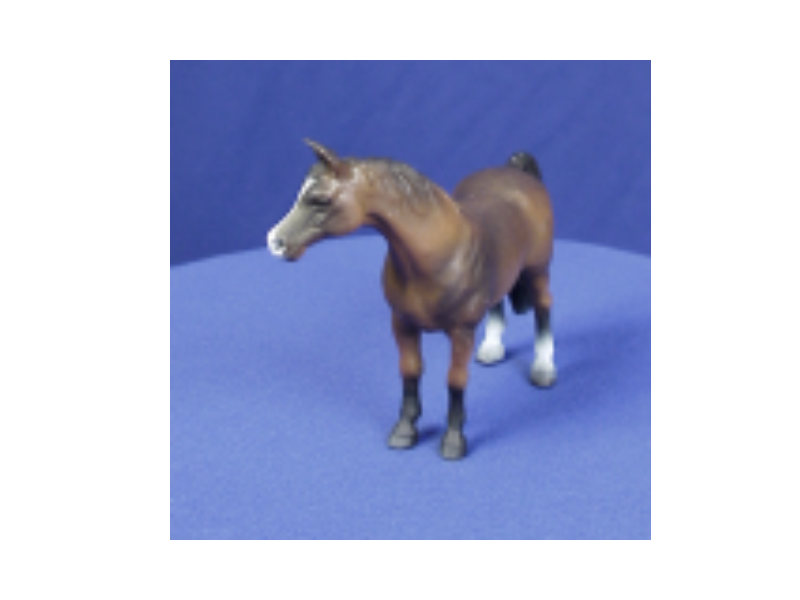} &
			\hspace{-1.0cm} 
			\includegraphics[width=0.22\linewidth]{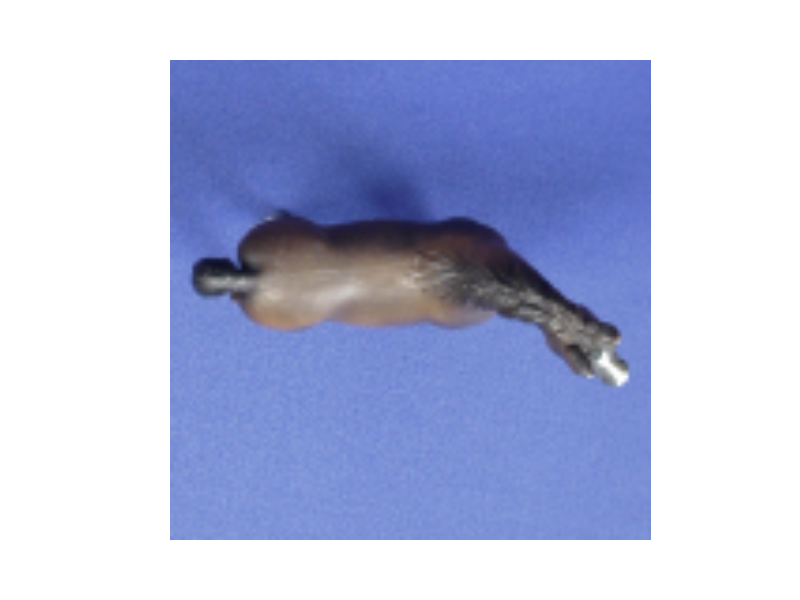} &
			\hspace{-1.0cm} 
			\includegraphics[width=0.22\linewidth]{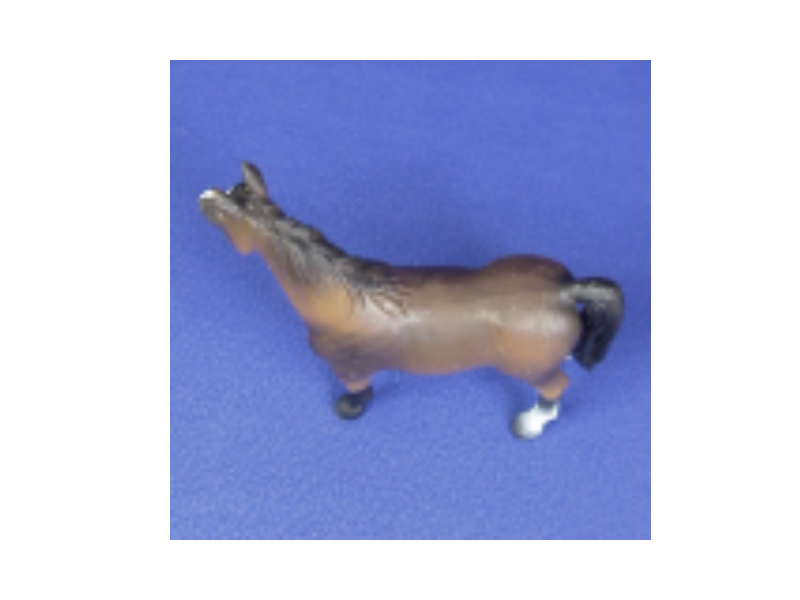} &
			\hspace{-1.0cm} 
			\includegraphics[width=0.22\linewidth]{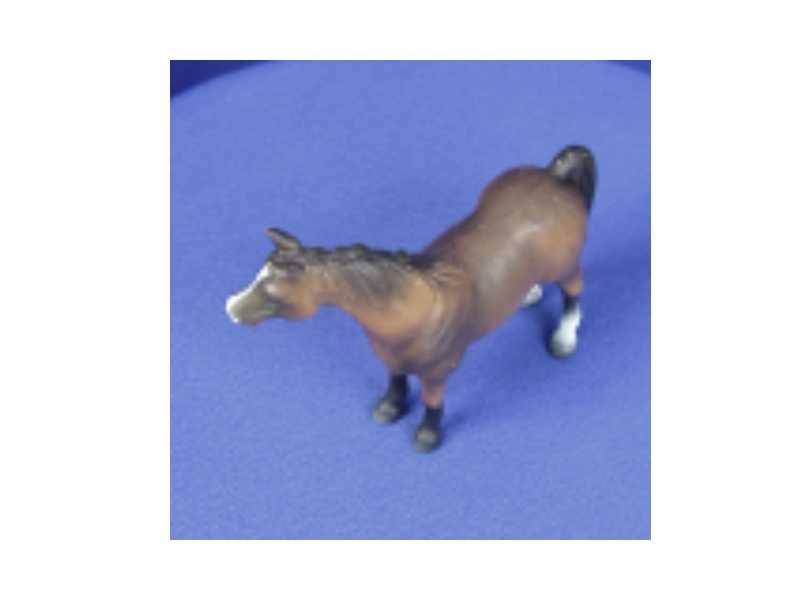} \\
			\vspace{-0.3cm} 	
			\hspace{-0.5cm} 	
			\includegraphics[width=0.22\linewidth]{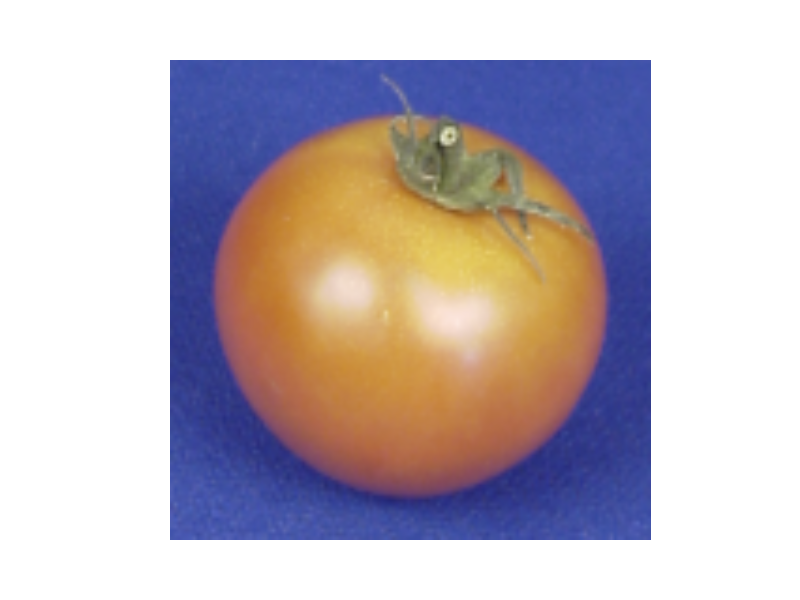}  &
			\hspace{-1.0cm} 
			\includegraphics[width=0.22\linewidth]{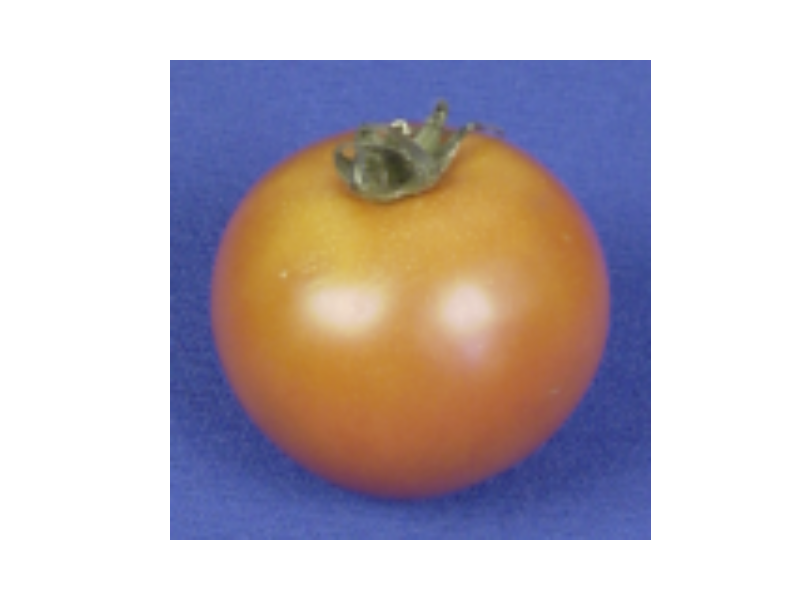} &
			\hspace{-1.0cm} 			
			\includegraphics[width=0.22\linewidth]{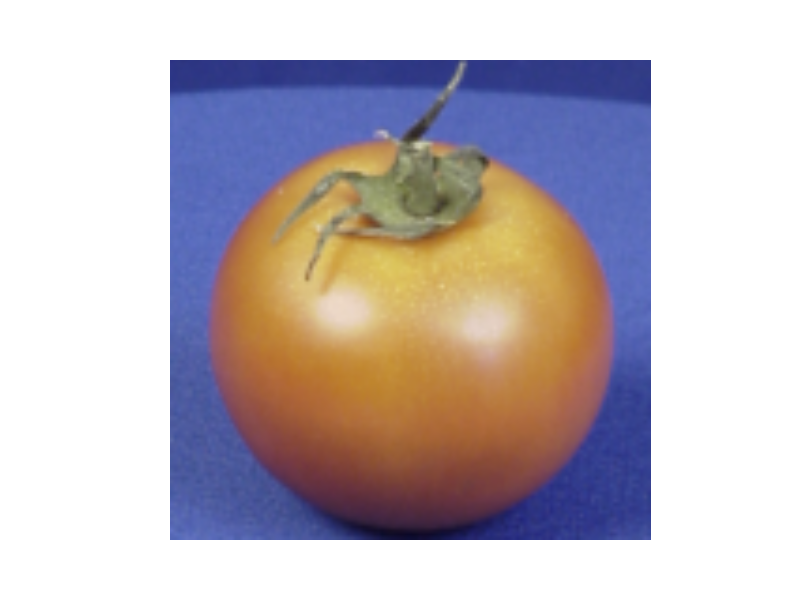} &
			\hspace{-1.0cm} 
			\includegraphics[width=0.22\linewidth]{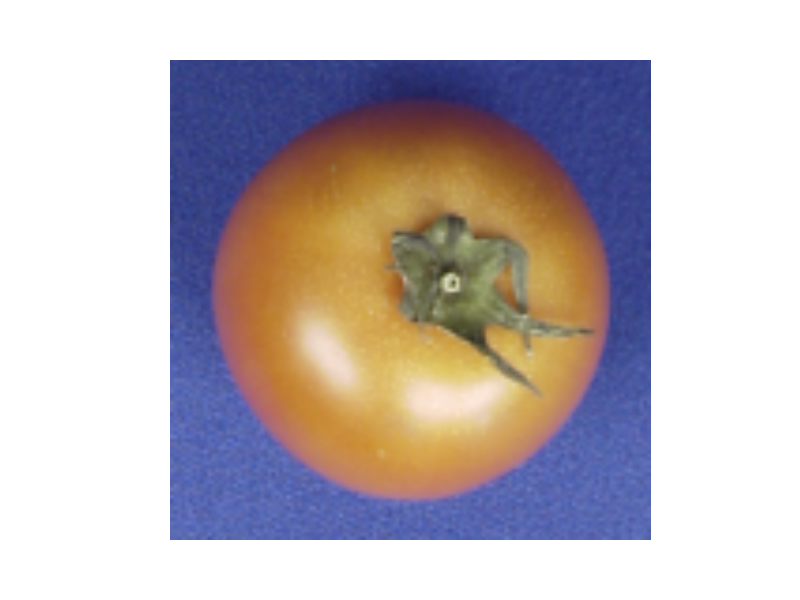} &
			\hspace{-1.0cm} 
			\includegraphics[width=0.22\linewidth]{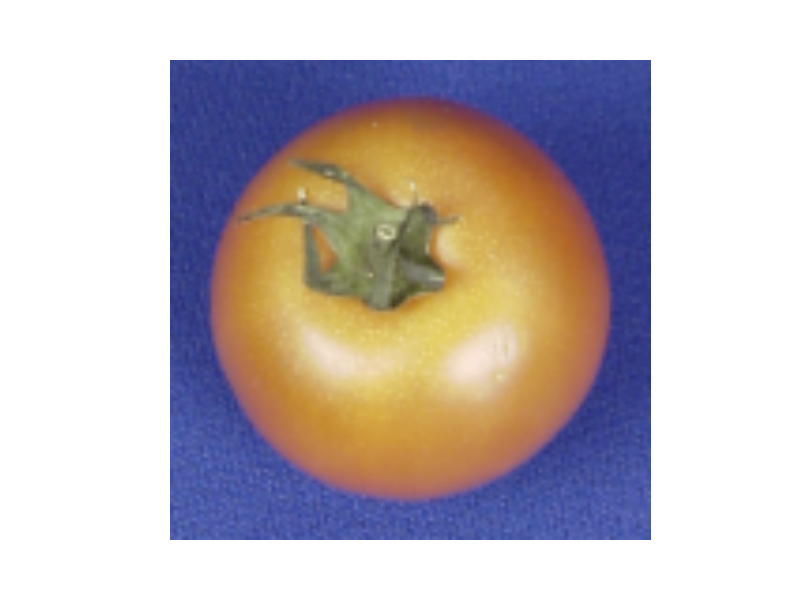} &
			\hspace{-1.0cm} 
			\includegraphics[width=0.22\linewidth]{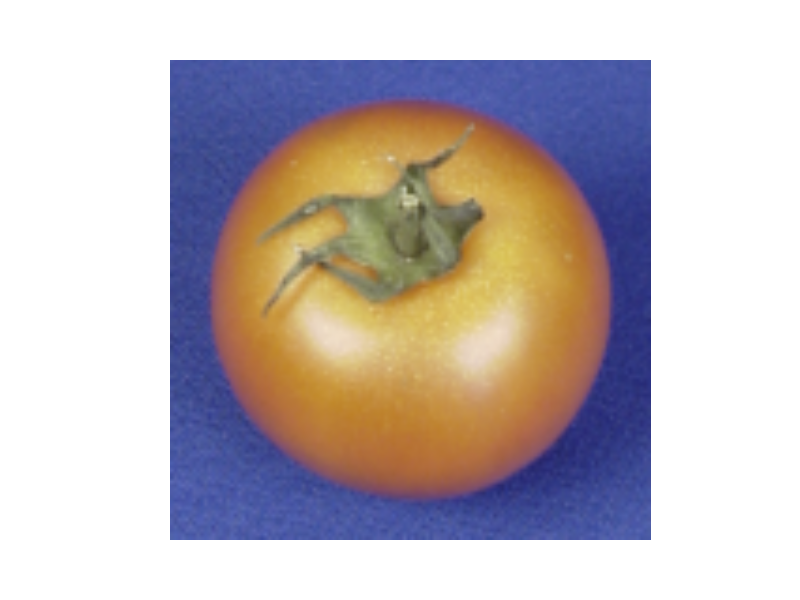} \\
			\hspace{+0.05cm} 
		\end{tabular}
		\caption{Sample images from ETH-80 dataset.}
		\label{fig:eth80}
	\end{figure}
	
	\subsection{TIGraNet Analysis}
	
	We analyze the performance of our new architecture on the MNIST-012 dataset. We first give some examples of feature maps that are produced by our network. We then illustrate the spectral kernels learned by our system, and discuss the influence of dynamic pooling operator.
	
	We first confirm the transformation invariant properties of our architecture. Even though our classifier is trained on images without any transformations, it is able to correctly classify rotated images in the test set, since our spectral convolutional layer learns filters that are equivariant to isometric transformations. We illustrate this in Fig.~\ref{fig:VFM}, which depicts several examples of feature maps $y_i^2$ from the second spectral convolutional layer for randomly rotated input digits in the test set. Each row of Fig.~\ref{fig:VFM} corresponds to images of a different digit, and we see that the corresponding feature maps are very close to each other (up to the image rotation) even when the rotation angle is quite large. This confirms that our architecture is able to learn features that are preserved with rotation, even if the training has been performed on non-transformed images. Despite important similarities in feature maps of rotated digits, one may however observe some slightly different values for the intensity. This can be explained by the fact that rotated versions of the input images may differ a bit from the original images due to interpolation artifacts.
	
	\begin{figure}
		\centering
		\includegraphics[width=0.9\linewidth]{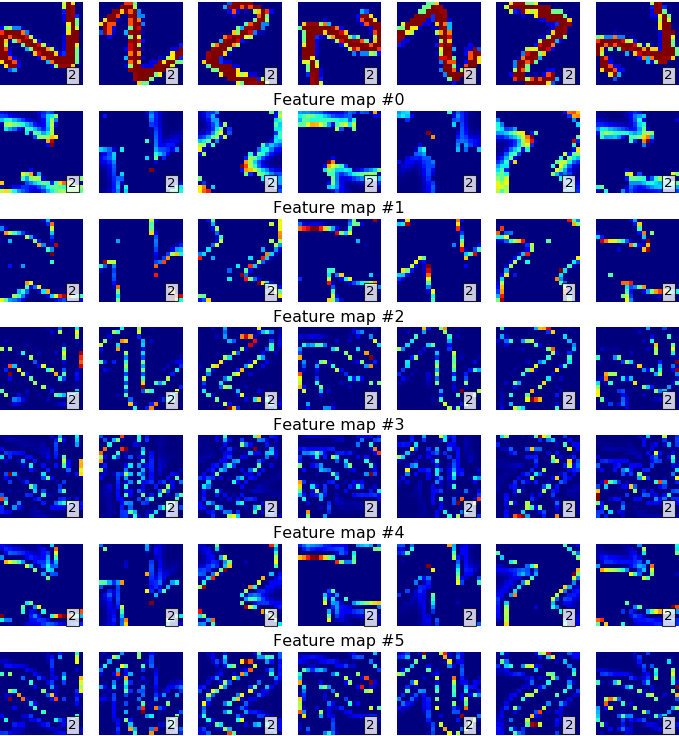} 
		\caption{{\bf Feature maps} from the second spectral convolutional layer for test images that are rotated versions of an image of the digit `2'. The predicted label for each of the images is further shown in the right bottom corner of each image.}
		\label{fig:example_fm}
		\label{fig:VFM}
	\end{figure}

	Fig.~\ref{fig:filter_ex} then shows the spectral representation of the kernels learned for the first two spectral convolutional layers of our network. As expected, the network learns filters that are quite different from each other in the spectral domain but that altogether cover the full spectrum. They permit to efficiently combine information in the different bands of frequency in the spectral representation of the input signal. Generally, the filters in the upper spectral convolutional layers are more diverse and represent more complicated features than those for the lower ones. 
	
	\begin{figure}[t!]
		\begin{tabular}{cc}
			\includegraphics[width=0.45\linewidth]{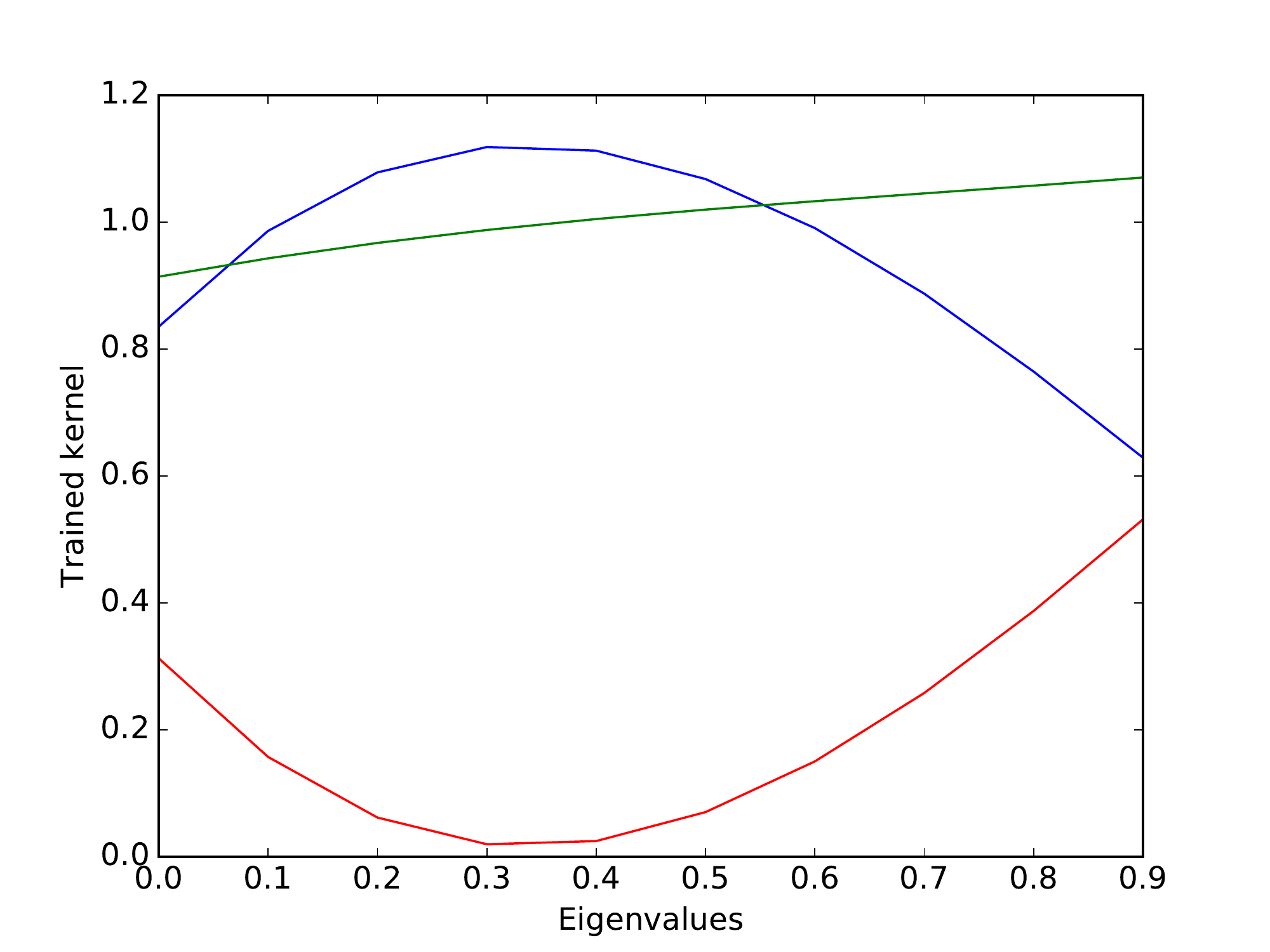} &
			\includegraphics[width=0.45\linewidth]{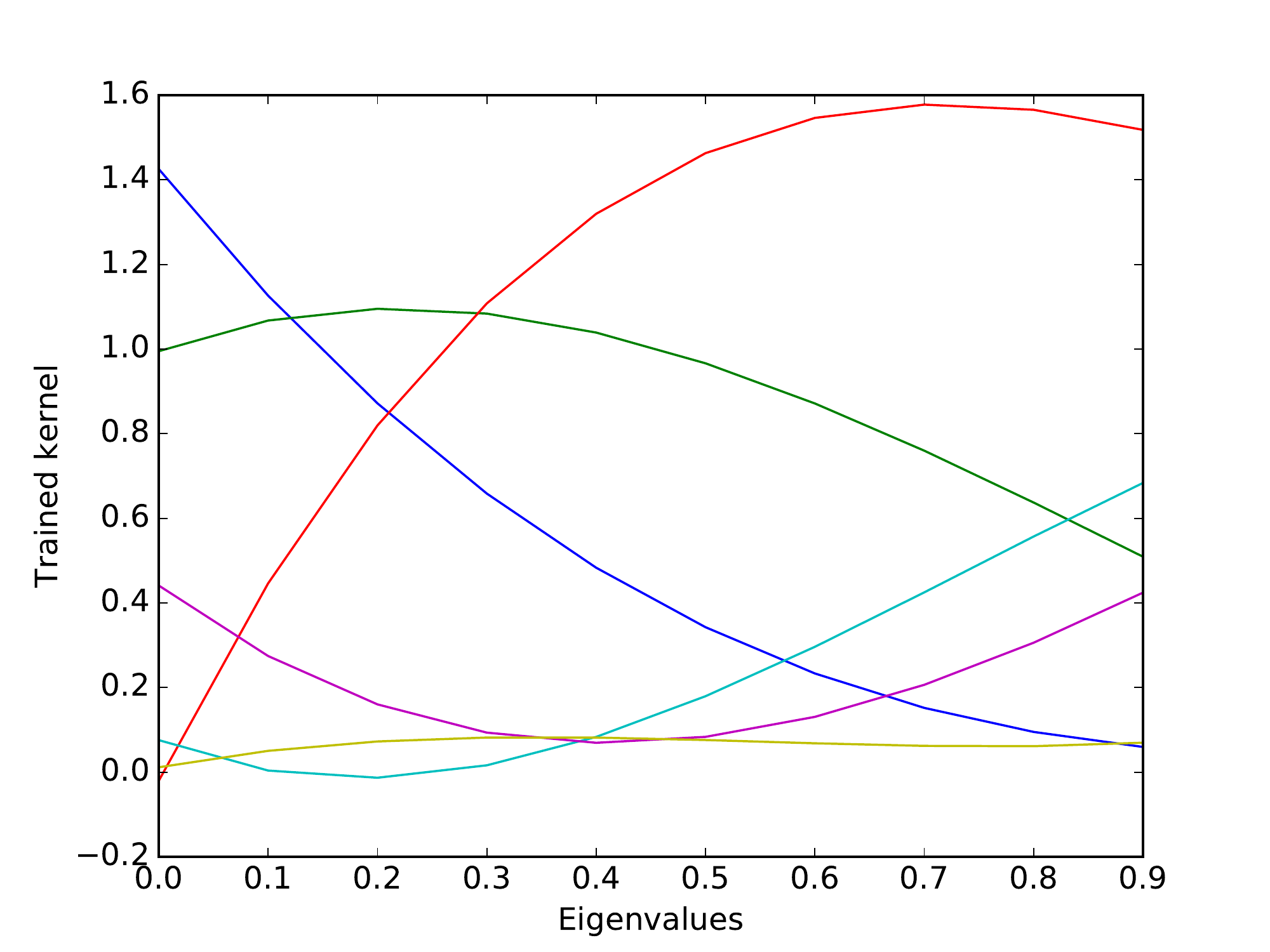}  \\
			a) & b) \\
		\end{tabular}
		\caption{{\bf Sample trained filters} in the spectral domain for (a) first and (b) second convolutional layers. Different colors represent different filters on each of the layers.} 
		\label{fig:filter_ex}
	\end{figure}
	\begin{figure}[t!]
		\centering
		\includegraphics[width=0.8\linewidth]{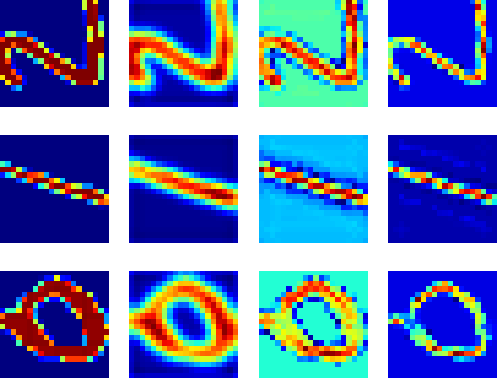}
		\caption{{\bf Feature maps after pooling} Each row shows different digits. The left most column depicts the original images, while the other columns show the features maps after dynamic pooling at the first, second and third layers respectively. The degree of the polynomial filters has been set to $M=3$ for each layer in this experiment.}
		\label{fig:feature_maps_examples}
	\end{figure}

	Further, we look at the influence of the new dynamic pooling layers in our architecture. Recall that dynamic pooling is used to reduce the network complexity and to focus on the representative parts of the input signal. Fig.~\ref{fig:feature_maps_examples} depicts the intermediate feature maps of the network for sample test images. We can see that after each pooling operation the signal is getting more and more sparse, while structure of the data that is important for discriminating images in different classes is preserved. That shows that our dynamic pooling operator is able to retain the important information in the feature maps constructed by the spectral convolutional layers. 

	\subsection{Influence of the image resolution on the quasi-equivariance}
	
	As we show in Section~\ref{s:discussion} the difference between filter responses of the polynomial filters $\Delta_\mathcal{F}$, defined in Eqs.~(\ref{eq:filt_diff_rot}) and~(\ref{eq:filt_diff}) depends on the resolution of the image $\Delta a, \Delta b$. To verify our theoretical result we run the following experiment. 
	
	We down-sample $400$ high resolution images from~\cite{bb:Agustsson_2017_CVPR_Workshops} using bicubic interpolation with several down-sampling factors $t\in\{2,3,4,5,6\}$. Thus, we obtain a set of images $y_t, t\in\{2,3,4,5,6\}$ with different resolutions. 
	
	For each of these images,  we apply isometric transformation $\mathcal{T}(y_t)$ and filters $\mF_i, i =1,...,20$ of degree $4$, with random coefficients $\alpha_i,m \in [-1, 1]$. Then, we calculate the following differences 
		\begin{equation}
		\mathbb{E}[\Delta_{\mF}](t) = \frac{1}{20} \sum_{i=1}^{20} \mathcal{T}^{-1}\big(\mF_i\big(\mathcal{T}(y_t)\big)\big) - \mF_i(y_t),
		\label{eq:dist}
		\end{equation}
	where $\mathcal{T}^{-1}$ is inverse transformation. In our experiments we apply the following transformations: 
	\begin{itemize}
	\item rotation by $\pi/18, \pi/9, \pi/6, \pi/4$;
	\item translation by $(0.1, 0.1),(0.2, 0.2),(0.3, 0.3),(0.4, 0.4)$ pixels;
	\end{itemize}
	and evaluate the mean equivariance gap $\mathbb{E}[\Delta_{\mF}](t)$ across various transformations and images. Fig.~\ref{fig:rotationresult} shows that $\mathbb{E}[\Delta_{\mF}](t)$ is growing with the increase of the  down-sampling factor~$t$, or equivalently the decrease in resolution of the images.  Also, we can see that translation gives smaller values of  $\mathbb{E}[\Delta_{\mF}](t), \forall t$ rather than rotation transformation, which confirms our results from Theorem~\ref{t:im_rot} and \ref{t:im_trans}, where it was shown that the upper bound on the equivariance gap depends on the maximum value of the second and third derivatives of the image signal $y$ for rotation and translation transformations respectively.
	  
	\begin{figure}[t!]
	\centering
	\begin{tabular}{rc}
		\rotatebox{90}{\scriptsize{Mean equivariance gap, $\mathbb{E}[\Delta_{\mF}](t)$}} & 
	\hspace{-0.3cm} \raisebox{-0.3cm}{\includegraphics[width=0.9\linewidth]{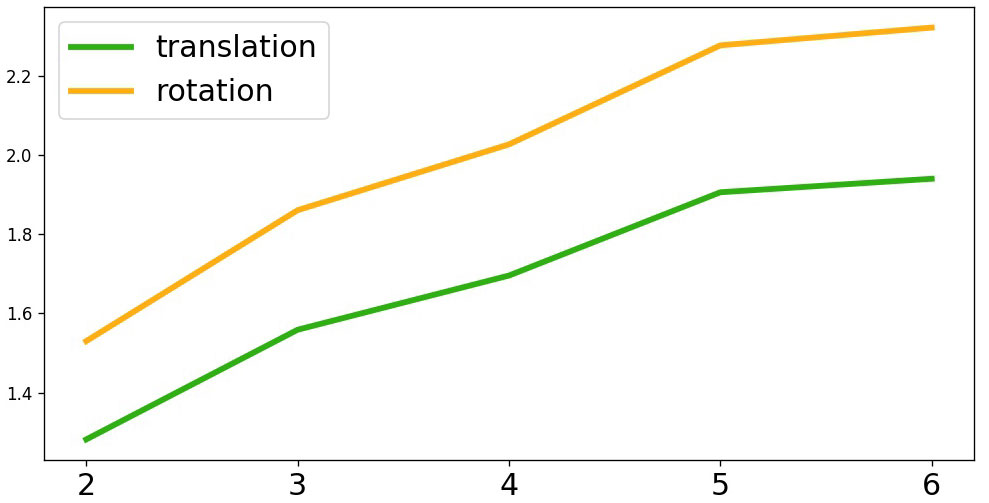}} \\
	& \scriptsize{Image down-sampling factor, $t$}\\
	\end{tabular}
	\caption{{\bf Influence of the image resolution on the quasi-equivariance property.}
		The $x$-axis illustrates the change in the image resolution that is defined by the down-sampling factor $t$ of the original image. The $y$-axis shows the mean equivariance gap $\mathbb{E}[\Delta_{\mF}](t)$ computed across a set of images with different transformations. The green and orange lines correspond to the set of image translations and rotations respectively. (best seen in color)}
	\label{fig:rotationresult}
\end{figure}

	
	\subsection{Performance evaluation}
	
	Here, we compare TIGraNet to state-of-the art algorithms for transformation-invariant image classification tasks, i.e., ConvNet~\cite{bb:lecun}, Spatial Transformer Network (STN)~\cite{bb:STN}, Deep Scattering (DeepScat)~\cite{bb:oyallon2015deep} and Harmonic Networks (HarmNet)~\cite{bb:harm}. Briefly, ConvNet is a classical convolutional deep network that is invariant to small image translations. STN compensates for image transformations by learning the affine transformation matrix. Further, DeepScat uses filters based on rich wavelet representation to achieve transformation invariance; however, it does not contain any parameters for the convolutional layers. Finally, HarmNet trains complex valued filters that are equivariant to signal rotations. For the sake of fairness in our comparisons, we use versions of these architectures that have roughly the same number of parameters, which means that each of the approaches learns features with a comparable complexity. For the DeepScat we use the default architecture. For HarmNet we preserve the default network structure, keeping the same number of complex harmonic filters, as the number of spectral convolutional filters that we have in TIGraNet.
	
	We first compare the performance of our algorithm to the ones of ConvNet and STN for the small digit dataset MNIST-012. The specific architectures used in this experiments are given in Table~\ref{tab:arch_part2}, where we use the following notations to describe it: C[$X_1$], P[$X_2$], FC[$X_3$] correspond to the convolutional, pooling and fully-connected layers respectively, with $X_1$ being the number of $3 \times 3$ filters, $X_2$ -- the size of the max-pooling area and $X_3$ -- the number of hidden units. ST[$X_4$] denotes the spatial transform layer with $X_4$ affine transformation parameters. W[$O, J$] and PCA[$X_5$] denote the parameters of DeepScat network with wavelet-based filters of order $O$ and maximum scales $J$, with dimension of the affine PCA classifier $X_5$. HRC[$X_6, X_7$] depicts the harmonic cross-correlation filter operating on the $X_7$ neighborhood with $X_6$ feature maps. HCN[$X_8$] is the complex nonlinearity layer of HarmNet with $X_8$ parameters. Finally, SC[$K_l$, $M$] is a spectral convolutional layer with $K_l$ filters of degree $M$, DP[$J_l$] is a dynamic pooling that retains $J_l$ most important values. Lastly, S[$K_{max}$] is a statistical layer with $K_{max}$ the maximum order of Chebyshev polynomials.
	
	
	\begin{table}[t!]
		\centering
		\begin{tabularx}{\linewidth}{ X c c c }
			\toprule
			& \scriptsize{Training set} & \scriptsize{Validation set}  & \scriptsize{Rotated test set} \\
			\midrule
			\multicolumn{4}{l}{\scriptsize{\bf{Training set with data augmentation}}}\\
			\qquad ConvNet &  99  &  94 &  $ 78 \pm 2.1$   \\
			\qquad STN & 100  &  97 & $ 93 \pm 0.97$   \\
			\midrule
			\multicolumn{4}{l}{\scriptsize{\bf{Training set without data augmentation}}}\\
			\qquad ConvNet & 100  &  100 &  $ 55 \pm 5$   \\
			\qquad STN & 100  &  98 & $ 50 \pm 5$   \\
			\qquad  TIGraNet & 98  &  97 & \bf{ 94 $\pm$ 0.42 }   \\
			\bottomrule
			\\
		\end{tabularx}
		\caption{Classification accuracy of ConvNet, STN and TIGraNet on MNIST-012. The methods are trained without and with transformed images. We average the performance of all the methods across 10 runs with different transformations of the test data.}
		\label{tab:comparison}
	\end{table}
	\begin{table}[t!]
		\centering
		\begin{tabularx}{\linewidth}{ X c c }
			\toprule
			& \scriptsize{MNIST-rot} & \scriptsize{MNIST-trans} \\
			\midrule
			\qquad ConvNet & 44.3 & 43.5  \\
			\qquad STN &  44.5 & 67.1 \\
			\qquad TIGraNet &  \bf{83.8} & \bf{79.6} \\
			\bottomrule
			\\
		\end{tabularx}
		\caption{Evaluation of the accuracy of the ConvNet, STN and TIGraNet on the MNIST-rot and MNIST-trans datasets. All the methods are trained on sets without transformed images.} 
		\label{tab:all_digit}
	\end{table}
	
	The results of this first experiment are presented in Table.~\ref{tab:comparison}. We can see that, if we train the methods on the dataset that does not contain rotated images, and test on the rotated images of digits, our approach achieves a significant increase in performance (i.e., $86\%$), due to its inherent transformation invariant characteristics. We further run experiments where a simple augmentation of the training set is implemented with randomly rotated images. This permits increasing the performance of all algorithms, as expected, possibly at the price of more complex training. Still, due to the rotation invariant nature of its features, TIGraNet is able to achieve higher classification accuracy than all its competitors. 

	We then run experiments on the MNIST-rot and MNIST-trans datasets. Note that both of them do not contain any isometric transformation in training and validation sets, but the test set contains transformed images. For all the methods we have used the architectures defined in Table~\ref{tab:arch_part2}. Table~\ref{tab:all_digit} shows that our algorithm significantly outperforms the competitor methods on both datasets due to its transformation invariant features.
	
	To further analyze the performance of our network we illustrate several sample feature maps for the different filters of the first two spectral convolutional layers of TIGraNet in Fig.~\ref{fig:il_big}, for the MNIST-rot and MNIST-trans datasets. We can see a few examples of misclassification of our network; for example, the algorithm predicts label `5' for the digit `6'. This mostly happens due to the border artifacts; if the digit is shifted too close to the border due to an isometric transformation, then the neighborhood of some nodes may change. 
	This problem can be solved by increasing the image borders or applying filters only to the central pixel locations. 

	\begin{figure}
		\centering
		\begin{tabular}{cc}
			\toprule
			\raisebox{2.5cm}{\rotatebox{90}{first layer}} &
			\hspace{-0.2cm} 
			\includegraphics[width=0.9\linewidth]{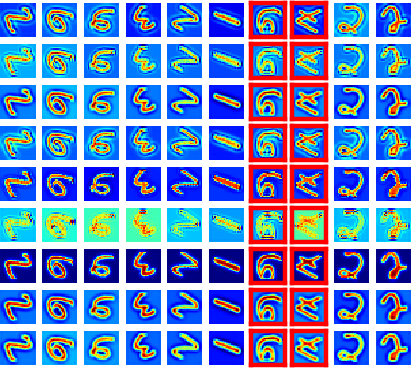} \\
			\midrule
			\raisebox{5.5cm}{\rotatebox{90}{second layer}} &
			\hspace{-0.2cm} 
			\includegraphics[width=0.9\linewidth]{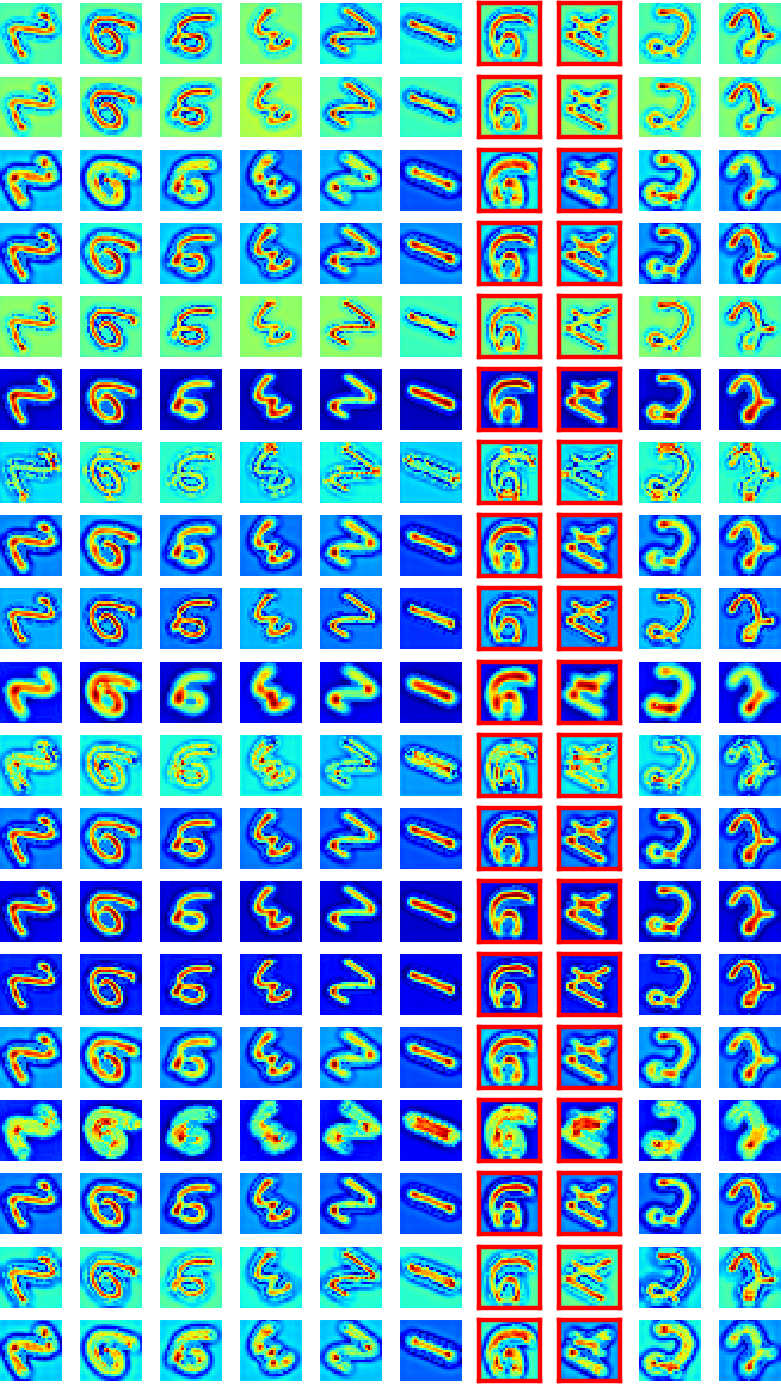} \\
			\bottomrule
		\end{tabular}
		\caption{{\bf Network feature maps visualization.} Each row shows the feature maps of different digits after the first and the second spectral convolutional layers. The misclassified images are marked by red bounding boxes. (best seen in color)
		}
		\label{fig:il_big}
	\end{figure}
	\begin{table}[!t]
		\centering
		\begin{tabularx}{\linewidth}{ X c }
			\toprule
			& Accuracy (\%) \\
			\midrule
			STN~\cite{bb:STN} &  45.1  \\
			ConvNet~\cite{bb:lecun} & 80.1   \\
			DeepScat~\cite{bb:oyallon2015deep} & 87.3 \\
			HarmNet~\cite{bb:harm} & 94.0   \\
			TIGraNet &  \bf{95.1} \\
			\bottomrule
			\\ 
		\end{tabularx}
		\caption{Performance evaluation for ConvNet, STN, DeepScat, HarmNet and TIGraNet on classification of images from the ETH-80 dataset.}
		\label{tab:res_eth}
	\end{table}
	
	Further, we evaluate the performance of our algorithm in more realistic settings where the objective is to classify images of objects that are captured from different viewpoints. This task requires having a classifier that is invariant to isometric transformations of the input signal. We run experiments on the ETH-80 dataset and compare the classification performance of TIGraNet to those of ConvNet, STN, DeepScat and HarmNet. The architectures of the different methods are described in Table~\ref{tab:arch_part2}. 
	
	Table~\ref{tab:res_eth} shows the classification results in this experiment. We can see that our approach outperforms the state-of-the-art methods due to its transformation invariant features. The closest performance is achieved by Harmonic Networks, since this architecture also learns equivariant features. It is important to note that the ETH-80 dataset contains less training examples than other publicly available datasets that are commonly used for the training of deep neural networks. This likely results in decrease of accuracy for methods such as~ConvNets and~STN. On the contrary, our method is able to achieve good accuracy even with small amounts of training data, due to its inherent invariance to isometric transformations.

	Finally, we run an additional experiment to show the influence of the amount of training data augmentation on the performance of our method. In this experiment we construct several training datasets $\mathcal{D}_i$ based on the training MNIST-rot images. To build each of these datasets $\mathcal{D}_i$ we randomly rotate its test images on one of $K_{\mathcal{D}_i}$ predefined angles. For example, $K_{\mathcal{D}_i}=4$ means that training images are randomly rotated by $0, \pi/2, \pi, 3\pi/2$ degrees. The Fig.~\ref{fig:randomrot} shows the performance of different methods on the test dataset containing MNIST-rot images rotated on random angle with respect to the number $K_{\mathcal{D}_i}$. As we can notice, our method trained on the dataset without any transformation outperforms the convolutional network even if it contains examples of many different rotations in the training data. However, the Spatial Transformer Network outperforms our method on the dataset with $K_{\mathcal{D}_i} > 4$. It happens because our network is inherently equivariant to graph isometric transformations that preserve graph structure (see Section~\ref{s:equivaricance}). Rotations on $45$ degrees, however, introduces interpolation artifacts. 
	Therefore, to be complete we also train our method on $D_i$, which allows our method to overcome these issues and again outperform other methods. 
	
	\begin{figure}[t!]
		\includegraphics[width=1\linewidth]{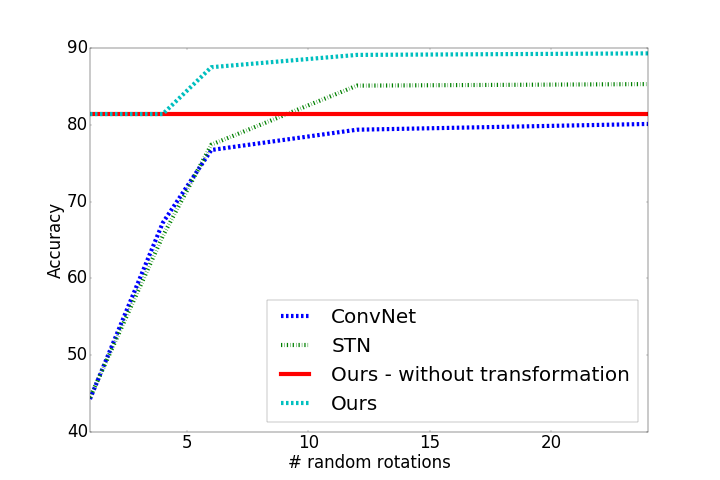}
		\caption{Performance of the different approaches depending on the number of random rotations in the training set. Red line corresponds to our method which is trained on the data without any augmentation.}
		\label{fig:randomrot}
	\end{figure}

	Overall, all the above experiments confirm the benefit of our transformation invariant classification architecture, which learns features that are invariant to transformation by construction. Classification performance improves with these features, such that the algorithm is able to reach sustained performance even if the training set is relatively small, or does not contain similar transformed images as in the test set. These are very important advantages in practice.

	\section{Conclusion}
	\label{s:conclusion}
	
	In this paper we present a new transformation invariant classification architecture, which combines the power of deep networks and graph signal processing, which allows developing filters that are equivariant to translation and rotation. A novel statistical layer further renders our full network invariant to the isometric transformations. This permits outperforming state-of-the-art algorithms on various illustrative benchmarks. Our new method is able to correctly classify rotated and translated images even if such transformed images do not appear in the training set. This confirms its high potential in practical settings where the training sets are limited but where the data is expected to present high variability.

	\bibliography{egbib_tigra}
	\bibliographystyle{IEEEtran}

\end{document}